\documentclass{article}
\usepackage{arxiv}
\usepackage[utf8]{inputenc} 
\usepackage[T1]{fontenc}    
\usepackage{hyperref}      
\usepackage{url}           
\usepackage{booktabs}       
\usepackage{amsfonts}       
\usepackage{nicefrac}       
\usepackage{microtype}      
\usepackage{xcolor}         
\usepackage{tcolorbox}
\usepackage{times}
\usepackage{epsfig}
\usepackage{graphicx}
\usepackage{amsmath}
\usepackage{amssymb}
\usepackage[]{amsthm}
\usepackage{bbm}
\usepackage{subfigure}
\usepackage{caption}
\usepackage[]{mathtools}
\usepackage{color}
\usepackage{wrapfig}
\definecolor{grey}{rgb}{0.33, 0.33, 0.33}
\newtcolorbox{mybox}{colback=blue!8!white,colframe=blue!8!white,left=1mm,top=-2mm,right=1mm,boxsep=0mm,width=10cm,before=\par\smallskip\centering,after=\par,
height=1cm}

\usepackage{graphicx}

\usepackage{tikz}
\usepackage{comment}
\usepackage{amsmath,amssymb} 
\usepackage{color}
\usepackage[accsupp]{axessibility}  

\usepackage{booktabs}

\usepackage{algorithm}
\usepackage{algorithmic}

\usepackage[capitalize]{cleveref}
\crefname{section}{Sec.}{Secs.}
\Crefname{section}{Section}{Sections}
\Crefname{table}{Table}{Tables}
\crefname{table}{Tab.}{Tabs.}

\newcommand{\squishlist}{
\begin{list}{{{\small{$\bullet$}}}}
{\setlength{\itemsep}{3pt}      \setlength{\parsep}{1pt}
\setlength{\topsep}{1pt}       \setlength{\partopsep}{0pt}
\setlength{\leftmargin}{1em} \setlength{\labelwidth}{1em}
\setlength{\labelsep}{0.5em} } }
\newcommand{\squishend}{  \end{list}  }
\usepackage{algorithm,algorithmic}

\newtheorem*{remark}{Remark}
\newcommand{\BR}{\mathbbm{1}}
\newcommand{\PG}{DuelGAN}%{{\small \textsf{PeerGAN}}}

\newcommand{\E}{\mathbb E}
\newcommand{\PP}{\mathbb P}

\newcommand{\dd}{\text{Duel-D}}

\newcommand{\pd}{p_{\text{duel}}}

\newtheorem{theorem}{Theorem}
\newtheorem{prop}{Proposition}
\newtheorem{lemma}{Lemma}

\newsavebox\MBox
\newcommand\Cline[2][red]{{\sbox\MBox{$#2$}%
  \rlap{\usebox\MBox}\color{#1}\rule[-1.2\dp\MBox]{\wd\MBox}{0.5pt}}}

\usepackage{newfloat}
\usepackage{listings}
\floatstyle{ruled}
\newfloat{listing}{tb}{lst}{}
\floatname{listing}{Listing}

\title{\textsf{DuelGAN}: A Duel Between Two Discriminators\\ Stabilizes the GAN Training}

\author{
Jiaheng Wei\thanks{Equal Contribution}\\
  UC Santa Cruz\\
  \texttt{jiahengwei@ucsc.edu} \\
  \And
  Minghao Liu\footnotemark[1] \\
  UC Santa Cruz\\
  \texttt{miu40@ucsc.edu} \\
  \And
  Jiahao Luo \\
  UC Santa Cruz\\
  \texttt{jluo53@ucsc.edu} \\
  \And
  Andrew Zhu \\
  UC Santa Cruz\\
  \texttt{angzhu@ucsc.edu} \\
  \And
  James Davis \\
  UC Santa Cruz\\
  \texttt{davis@cs.ucsc.edu} \\
  \And
 Yang Liu\thanks{Correspondence to {yangliu@ucsc.edu}}\\
  UC Santa Cruz\\
  \texttt{yangliu@ucsc.edu} \\
}

\begin{document}
\maketitle

\begin{abstract}
  In this paper, we introduce \PG{}, a generative adversarial network (GAN) solution to improve the stability of the generated samples and to mitigate mode collapse. Built upon the Vanilla GAN's two-player game between the discriminator $D_1$ and the generator $G$, we introduce a peer discriminator $D_2$ to the min-max game. Similar to previous work using two discriminators, the first role of both $D_1$, $D_2$ is to distinguish between generated samples and real ones, while the generator tries to generate high-quality samples which are able to fool both discriminators. 
     Different from existing methods, we introduce a duel between $D_1$ and $D_2$ to discourage their agreement and therefore increase the level of diversity of the generated samples. This property alleviates the issue of early mode collapse by preventing $D_1$ and $D_2$ from converging too fast. We provide theoretical analysis for the equilibrium of the min-max game formed among $G,D_1,D_2$. We offer convergence behavior of \PG{} as well as stability of the min-max game. It's worth mentioning that \PG{} operates in the unsupervised setting, and the duel between $D_1$ and $D_2$ does not need any label supervision. Experiments results on a synthetic dataset and on real-world image datasets (MNIST, Fashion MNIST, CIFAR-10, STL-10, CelebA, VGG, and FFHQ) demonstrate that \PG{} outperforms competitive baseline work in generating diverse and high-quality samples, while only introduces negligible computation cost.
\end{abstract}

\section{Introduction}

Vanilla GAN (Generative Adversarial Nets \cite{gan}) proposed a data generating framework through an adversarial process which has achieved great success in image generation \cite{gan,progan,msgan,autogan,biggan,crgan,distgan,sngan,gradientgan,mmdgan,iwgan,dieng2019prescribed}, image translation \cite{cyclegan,stackgan,tacgan,resolutiongan}, and other real-life applications \cite{cartoongan,pose,emoji,image_edit,face_aging,gpgan,super_resolution,impainting1,impainting2,face_completion,videos}. However, training Vanilla GAN is usually accompanied with a number of common problems, for example, vanishing gradients, mode collapse and failure to converge. Unfortunately, none of these issues have been completely addressed. There is a large amount of follow up work on Vanilla GAN. Due to space limitations, we only discuss the two most related lines of works. 

\subsection{Stable and Diverse GAN Training}
Several stabilization techniques have been implemented in GAN variants. Modifying architectures is the most extensively explored category. Radford et al. \cite{DCGAN} make use of convolutional and convolutional-transpose layer in training the discriminator and generator. Karras et al. \cite{progan} adopt a hierarchical architecture and trains the discriminator and generator with progressively increasing size. Huang et al. \cite{sgan} proposed a generative model which consists of a top-down stack of GANs. Chen et al. \cite{infogan} split the generator into the noise prior and also latent variables. The optimization task includes maximizing the mutual information between latent variables and the observation. Designing suitable loss functions is another favored technique. Successful designs include $f$-divergence based GAN \cite{f_gan,mao2017least} (these two approaches replace loss functions of GAN by estimated variational $f$-divergence or least-square loss respectively), introducing auxiliary terms in the loss function \cite{UnrolledGAN} and integral probability metric based GAN \cite{arjovsky2017wasserstein,iwgan,dragan,lsgan}. A detailed survey of methods for stabilizing GANs exists \cite{stablegan_survey}. 

\subsection{Multi-Player GANs}
Multi-player GANs explore the situation where there are multiple generators or multiple discriminators. The first published work to introduce multiple discriminators to GANs is multi-adversarial networks, in which discriminators can range from an unfavorable adversary to a forgiving teacher~\cite{gman}. Nguyen et al. \cite{D2GAN} formulate D2GAN, a three-player min-max game which utilizes a combination of Kullback-Leibler (KL) and reverse KL divergences in the objective function and is the most closely related to our work. Albuquerque et al. \cite{hgan} show that training GAN variants with multiple discriminators is a practical approach even though extra capacity and computational cost are needed.  Employing multiple generators and one discriminator to overcome the mode collapse issue and encourages diverse images has also been proposed \cite{mgan,madgan}. 

In contrast to the above existing work, we demonstrate the possibility of improving GAN training with a computationally light modification by adding only one competing discriminator. We introduce a duel game among two discriminators and demonstrate the benefits of doing so in stabilizing and diversifying the training. 

Our main contributions summarize as follows:
\squishlist
    \item We introduce a duel  between two discriminators to encourage diverse predictions and avoid early failure. The intuition is that predictions with high consensus will be discouraged, and effectively both discriminators are rewarded for having diverse predictions. The introduced game between the two discriminators results in a different convergence pattern for the generator. 
    \item Theoretically, we derive the equilibrium for discriminators and the generator. We show how \PG{} alleviates the vanishing gradient issue and mode collapse intuitively and empirically. We derive evidence for how the peer discriminator helps the dynamics of the learning. In addition, we demonstrate that if the peer discriminator is better than a random guess classifier, the intermediate game and the objective function in \PG{} are stable/robust to a bad peer discriminator.
    \item Experimental results on a synthetic dataset validate that \PG{} addresses mode collapse. Results on real datasets demonstrate that \PG{} generates high-quality image samples compared with baseline works. Besides, the introduced duel-game could also be viewed as a regularizer which complements well with existing methods and further improves the performance.
\squishend

\section{Background}
We first review Vanilla GAN and D2GAN, which are the most relevant to understanding our proposed \PG{}.

\subsection{Vanilla GAN \cite{gan}}
Let $\{x_i\}_{i=1}^{n}\subseteq \mathcal{X}$ denote the given training dataset drawn from the unknown distribution $p_{\text{data}}$. Traditional GAN formulates a two-player game: a discriminator $D$ and a generator $G$.  To learn the generator $G$'s distribution over $\mathcal{X}$, $G$ maps a prior noise distribution $p_z(z)$ to the data space. $\forall x\in \mathcal{X}$, $D(x)$ returns the probability that $x$ belongs to $p_{\text{data}}$ rather than $p_g$, where $p_g$ denotes the distribution of $G(z)$ implicitly defined by $G$. GAN trains $D$ to maximize the probability of assigning the correct label to both training samples and those from the generator $G$. Meanwhile, GAN trains $G$ to minimize $\log(1-D(G(z)))$.
\begin{align}
\label{equ:gan}
    \min_{G}\max_{D}V(D, G)=&\E_{x\thicksim p_{\text{data}}}[\log D(x)]+\E_{z\thicksim p_{z}}\Big[\log \Big(1-D\big(G(z)\big)\Big)\Big].
\end{align}

\subsection{D2GAN \cite{D2GAN}}
D2GAN is the most closely related method to \PG{}. This three-player game aims to solve the mode collapse issue and the optimization task is equivalent to minimizing both KL divergence and Reverse-KL divergence between $p_{\text{data}}$ and $p_g$.
The formulation of D2GAN comes as follows:
\begin{align}
\label{equ:d2gan}
        \min_{G}\max_{D_1, D_2}V(D_1, D_2, G)=&\alpha \cdot \E_{x\thicksim p_{\text{data}}}[\log D_1(x)]+
        \Cline{\E_{z\thicksim p_{z}}\big[-D_1\big(G(z)\big)\big]}
        \notag\\&+\Cline{\E_{x\thicksim p_{\text{data}}}[- D_2(x)]}+\beta \cdot \Cline{\E_{z\thicksim p_{z}}\big[\log D_2\big(G(z)\big)\big]}.
\end{align}
Given a sample $x$ in data space, $D_1(x)$ rewards a high score if $x$ is drawn from $p_{\text{data}}$, and gives a low score if generated
from the generator distribution $p_g$. In contrast, $D_2(x)$ returns a high score for $x$ generated from $p_g$ and gives a low score for a sample drawn from $p_{\text{data}}$. Our work is similar to D2GAN in containing a pair of discriminators, however instead of discriminators with different goals, we use identical discriminators and introduce a duel/competition between them. 
\section{\PG{}: A Duel Between Two Discriminators}

In this section, we first give the formulation and intuition of \PG{}. Then we will present the equilibrium strategy of the generator and the discriminators.
\subsection{Formulation}
\label{sec:formulate}
\begin{figure*}[t]
\vspace{-0.18in}
    \centering
    {\includegraphics[width=\textwidth]{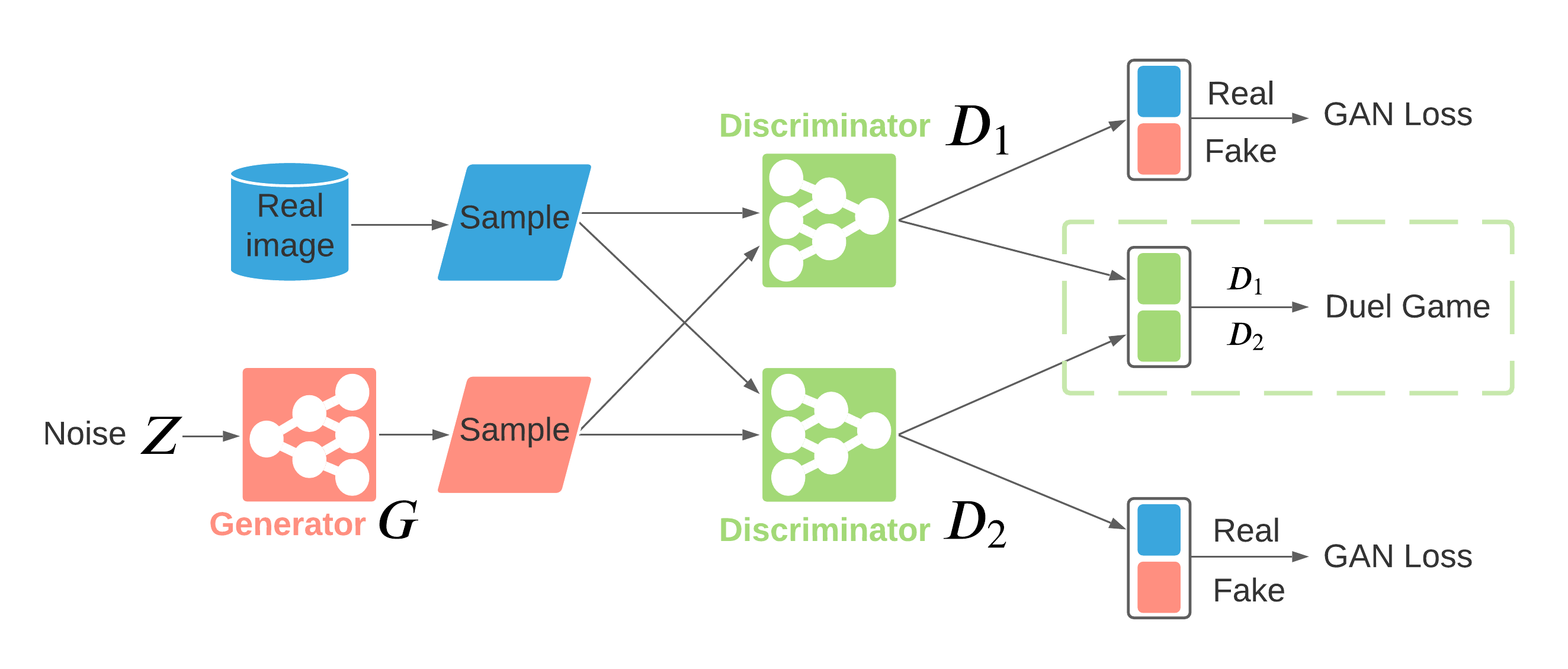}
    }
        \vspace{-5pt}
        \caption{Illustration of the proposed \PG{}. Compared with Vanilla GAN, \PG{} has one more identical discriminator and a Duel Game between two discriminators. The introduced Duel Game induces diversified generated samples by discouraging the agreement between $D_1$ and $D_2$. In D2GAN, although both discriminators are trained with different loss functions, they do not interfere with each other in the training.
        \vspace{-10pt}
    }
    \label{Fig:vis_dif}
\end{figure*}
Similar to related works, we assume that the data follows the distribution $p_{\text{data}}$, our ultimate goal is to achieve $p_{g}=p_{\text{data}}$ where $p_{g}$ is the generator's distribution. \PG{} formulates a three-player game which consists of two discriminators $D_1$, $D_2$ and one generator $G$. Denote by $\pd$ an equal mixture of $p_{\text{data}}$ and $p_{g}$, $\forall x$: $
\pd(x)=\frac{p_{\text{data}}(x)+p_{g}(x)}{2}.$
Recall that $p_z$ denotes the prior noise distribution, now we are ready to formulate the min-max game of \PG{} as follows:
\begin{tcolorbox}[colback=grey!10!white,colframe=grey!10!white]
\begin{align}
\label{equ:peergan}
    &\min_{G} \max_{D_1, D_2}\mathcal{L}(D_1, D_2, G)\notag\\
    =&\min_{G} \max_{D_1, D_2} \E_{x\thicksim p_{\text{data}}}\left[\log D_1(x)\right]+\Cline{\E_{x\thicksim p_{\text{data}}}\left[\log D_2(x)\right]}\notag\\
    &+\beta \cdot \Cline{\dd}+\Cline{\E_{z\thicksim p_{z}}\left[\log \left(1-D_1\left(G(z)\right)\right)\right]}+\Cline{\E_{z\thicksim p_{z}}\left[\log \left(1-D_2\left(G(z)\right)\right)\right]},
\end{align}
\end{tcolorbox}
where $\dd$ introduces the duel (a peer competition game) among $D_1,D_2$, defined as:
\begin{align}
    \label{equ: peer_d}
     \dd=\E_{x\thicksim \pd}&\Bigg[\underbrace{\ell\Big(D_1(x),\BR\big(D_2(x)>\dfrac{1}{2}\big)\Big)}_\text{Term \text{1a}}-\alpha \cdot \underbrace{\ell\Big(D_1(x_{p_1}),\BR\big(D_2(x_{p_2})>\dfrac{1}{2}\big)\Big)\Bigg]}_\text{Term \text{1b}}\notag\\
     +\E_{x\thicksim \pd}&\Bigg[\underbrace{\ell\Big(D_2(x),\BR\big(D_1(x)>\dfrac{1}{2}\big)\Big)}_\text{Term \text{2a}}-\alpha \cdot \underbrace{\ell\Big(D_2(x_{p_1}),\BR\big(D_1(x_{p_2})>\dfrac{1}{2}\big)\Big)\Bigg]}_\text{Term \text{2b}}.
\end{align}
In $\dd$, $x_{p_1}$ and $x_{p_2}$ are drawn randomly from $\pd$ and that \underline{$x,x_{p_1}$ and $x_{p_2}$} \underline{are independent with each other.} $\BR(\cdot)$ is the indicator function, $\alpha,\beta\in [0,1]$ are hyper-parameters controlling the disagreement level and the weight of the competition game between two discriminators, respectively. $\ell$ is an evaluation function, for simplicity, we adopt $\ell = \log(\cdot)$, as commonly used in other terms in the min-max game. Thus, we have:
\begin{align}
  \ell(D_i(x),y) =
    \begin{cases}
      \log\big(D_i(x)\big) & \text{if  $y=1$;}\\
      \log\big(1-D_i(x)\big) & \text{if $y=0$.}
    \end{cases}       
\end{align}
To clarify the differences among Vanilla GAN \cite{gan}, D2GAN \cite{D2GAN} and \PG{}, we use an workflow to illustrate in Figure \ref{Fig:vis_dif}. The key differences in \PG{}'s formulation can be summarized as follows: 
\squishlist
    \item Compared with Vanilla GAN (see Eqn.(\ref{equ:gan})), \PG{} (see Eqn.(\ref{equ:peergan})) introduces a peer discriminator $D_2$ which has the same objective function as $D$ appeared in Eqn.(\ref{equ:gan}). An intermediate duel game $\dd$ is added which will be explained below.
    \item The difference between D2GAN (see Eqn.(\ref{equ:d2gan})) and \PG{} is highlighted with the underscores in red. Primarily, there is no interaction between discriminators in D2GAN, while our $\dd$ term introduces another duel game between the discriminators, which we explain below. In addition to \dd, the objective function in \PG{} encourages both discriminators to fit perfectly on both training samples and generated samples. While in D2GAN, one discriminator fits overly on training samples, the other fits overly on generated samples.  
\squishend
\subsection{Competition Introduced by \dd} 
$\dd$ bridges $D_1$ and $D_2$ by introducing 4 terms specified in Eqn.(\ref{equ: peer_d}). Since we do not expect arbitrarily different discriminators, and both $D_i$s should play against the generator $G$, \text{Term \text{1a}} and \text{Term \text{2a}} encourage agreements between $D_1$ and $D_2$. With only these two terms, $D_1$ and $D_2$ will eventually be encouraged to converge to agree with each other. Mode collapse issue remains a possibility. \PG{} introduces \text{Term \text{1b}} and \text{Term \text{2b}} to the objective function which punish $D_1$ and $D_2$ from over-agreeing with each other (where the duel happens), especially at the early phase of training.  Particularly, the Term 1b and 2b are evaluating the agreements of $D_1$ and $D_2$ on two entirely independent samples $x_{p_1}, x_{p_2}$. Because of the independence, the two discriminators' predictions should not match with high probability. Note that the calculation of $\dd$ does not need label supervisions, which distinguishes our work from other works that introduces multiple discriminators but would require additional label supervisions \cite{tacgan}.

We provide more details of our intuition as well as theoretical evidences of this property in Section \ref{sec:prop}. 

\subsection{The Max Game of Discriminators}
\label{sec:max}
Denote the true label of $x$ as $y=1$ if $x$ comes from $p_{\text{data}}$, otherwise, $y=0$. For any given generator $G$, let us first analyze the best responding/optimal discriminator $D_{i,G}^*(x),~i\in {1,2}.$ We define the following quantities: 
\begin{align}
    r_{i, G}(x) &:= \mathbb{P}_{x\thicksim \pd}\Big(\BR\big(D_{i}(x)>\dfrac{1}{2}\big)=1\Big), \quad p_{i,G}:=\mathbb{E}_{x\thicksim \pd} [r_{i, G}(x)],
\end{align}
where $r_{i, G}(x)$ represents the probability/confidence of $x$ being categorized as the real data by $D_i$ and $p_{i,G}$ is the expectation of $r_{i, G}(x)$ for $x\sim \pd$.
Let $\hat{r}^*_{i, G}(x):=r_{i,G}(x)-\alpha\cdot p_{i,G}$.  Given discriminator $D_{i}$, when there is no confusion,  \underline{we use $D_{j}$ to denote the peer discriminator without telling $j\neq i$ in later sections.}

\begin{prop}\label{prop: opt_d}
For $G$ fixed, denote by $w:=\beta\cdot (1-\alpha)$, the optimal discriminators $D_{1}, D_{2}$ are given by:

\begin{align}
    D_{i,G}^*(x)=\dfrac{p_{\text{data}}(x)+\beta \cdot \hat{r}_{j, G}^*(x)\cdot \pd(x)}{p_{\text{data}}(x)+p_{g}(x)+w\cdot \pd(x)}, \quad i=1,2.
\end{align}
\end{prop}

\subsection{The Min Game of the Generator}
Remember that the training objective for $D_{i}$ can be interpreted as maximizing the log-likelihood for estimating the conditional probability $\PP(Y=y|x)$ where $Y$ indicates whether $x$ comes from $p_{\text{data}}$ (with $y=1$) or from $p_{g}$ (with $y=0$). With the introduce of Duel Game, the distributions $p_{\text{data}}$ and $p_{g}$ in the Vanilla GAN got changed due to the appearance of $\pd$. Thus, we define the corresponding updated distributions in \PG{} w.r.t. discriminator $D_i$ as $p_{\text{data}_i}$ and $p_{g_i}$, respectively. For a clean presentation, we defer the exact form of $p_{\text{data}_i}$, $p_{g_i}$ in Appendix (Eqn.(\ref{eqn:new_dis})).

Denote $C(G):=\max_{D}\mathcal{L}(G, D_{1}, D_{2})$, the inner-max game ($C(G)$) can be rewritten as (straightforward in the proof of Proposition \ref{prop: opt_d} which is available in the Appendix \ref{app:p1}): 
\begin{align}
   C(G)=&\E_{x\thicksim p_{\text{data}_1}}[\log D_{1, G}^{*}(x)]+\E_{x\thicksim p_{g_1}}\big[\log \big(1-D_{1, G}^{*}(x)\big)\big] \notag\\
     +&\E_{x\thicksim p_{\text{data}_2}}[\log D_{2, G}^{*}(x)]+\E_{x\thicksim p_{g_2}}\big[\log \big(1-D_{2, G}^{*}(x)\big)\big].
\end{align}

\begin{theorem}
\label{thm: global_min}
When $\alpha=0, r_{j,G}(x)=\frac{1}{2}$, the global minimum of the virtual training criterion $C(G)$ is achieved if and only if $p_{\text{data}}=p_{g}$. At this point, $C(G)$ achieves the value of $-\log{16}$.
\end{theorem}

\subsection{When $r_{j,G}(x)=\frac{1}{2}$?} 
Note that $r_{j,G}(x)$ is merely representing the probability that $D_j$ classifies $x$ to be real samples, $p_{j,G}$ is the probability that $D_{j}$ classifies a random sample as the real one. Without loss of generality, we assume real and generated samples are of uniform/equal prior. At the very beginning of the training process, the discriminator can do well in distinguishing real or generated samples, since the generator at this time generates low-quality samples. In this case, $r_{j,G}(x)$ is supposed to approach its max/min value, for example, $r_{j,G}(x)\to 0$ if $x$ is from generated samples, and otherwise, $r_{j,G}(x)\to 1$. During the training process, the generator progressively tries to mislead the predictions made by discriminators, which means the discriminator can not decide whether the sample is being fake or real. Thus, $r_{j, G}(x)\to \frac{1}{2}$. At this time, for $\alpha= 0, i=1,2$, we have:
\begin{align}
    D_{i,G}^*(x)&=\dfrac{p_{\text{data}}(x)+\beta \cdot \hat{r}_{j,G}^*(x)\cdot \pd(x)}{p_{\text{data}}(x)+p_{g}(x)+\beta\cdot \pd(x)}\rightarrow \dfrac{p_{\text{data}}(x)+\dfrac{\beta}{2} \cdot \pd(x)}{p_{\text{data}}(x)+p_{g}(x)+\beta\cdot \pd(x)}.
\end{align}
This allows us to rewrite $\frac{C(G)}{2}$ as: $\E_{x\thicksim p_{\text{data}_i}}\left[\log \frac{p_{\text{data}}(x)+\frac{\beta}{2}\cdot \pd(x)}{p_{\text{data}}(x)+p_{g}(x)+\beta\cdot \pd(x)}\right]+\E_{x\thicksim p_{g_i
     }}\left[\log \frac{p_{g}(x)+\frac{\beta}{2}\cdot \pd(x)}{p_{\text{data}}(x)+p_{g}(x)+\beta\cdot \pd(x)}\right].$ Our subsequent proof is then based on the above reformulation.

We summarize the overall DuelGAN algorithm in Algorithm \ref{m:alg1}. In experiments, we train $G$ to minimize $\log (1-D_i(G(z)))$ which is equivalent to maximizing $\log D_i(G(z))$.

\begin{algorithm}
% \small
\caption{\PG{}}\label{m:alg1}
\begin{algorithmic}[1]
\STATE \textbf{Input}: two discriminators $D_1, D_2$, generator $G$, training samples $\{x_i\}_{i=1}^{n}$, weights $\alpha, \beta$.\\
\STATE \textbf{For} number of training iterations \textbf{do}\\
 \setlength \parindent{10pt}  \textbf{For} $1$ to $k$ steps \textbf{do}\\
 \setlength \parindent{20pt} 
 \begin{itemize}
     \item Sample mini-batch of $m$ noise samples $Z=\{z_{1}, ..., z_{m}\}$ from noise prior $p_z$.
     \item Sample mini-batch of $m$ samples $X=\{x_{1}, ..., x_{m}\}$ from data generating distribution $p_{\text{data}}(x)$.
     \item Combine two subsets $T:=X \cup Z$, and denote by $T=\{t_1,...,t_{2m}\}$. 
     \item Update discriminator $D_i (i\in \{1, 2\})$ by ascending the stochastic gradient:
   { \begin{align}
    &\nabla_{\theta_{d_i}}\dfrac{1}{m}\sum_{i=1}^{m}\Big[\log D_{i}(x_{i})+\log \Big(1-D_{i}\big(G(z_{i})\big)\Big)\Big]\notag\\
    & +\dfrac{\beta}{2m}\sum_{j=1}^{2m}  \Bigg[ \ell_{\text{CE}}\Bigg(D_i(t_{j}), \BR\Big(D_{j}(t_{j})>\dfrac{1}{2}\Big)\Bigg)-\alpha \cdot  \ell_{\text{CE}}\Bigg(D_i(t_{p_1}), \BR\Big(D_{j}(t_{p_2})>\dfrac{1}{2}\Big)\Bigg)\Bigg],
    \end{align}}
    where $t_{p_1}, t_{p_2}$ are randomly selected (with replacement) samples from $T$.
    \item Update $G$ by descending its stochastic gradient:
  { \begin{align}
    \nabla_{\theta_{g}}\dfrac{1}{m}\sum_{i=1}^{m}\Big[\log \Big(1-D_{1}\big(G(z_{i})\big)\Big)
    +\log \Big(1-D_{2}\big(G(z_{i})\big)\Big)\Big].
    \end{align}}
 \end{itemize}
\end{algorithmic}
\end{algorithm}

\section{Properties of \PG{}}
\label{sec:prop}

In this section, we first illustrate how \PG{} alleviates common issues in GAN training, for example, the vanishing gradients issue and  the mode collapse issue. Then we present properties of \PG{} including its stability guarantee and converging behavior. 

\subsection{\PG{} and Common Issues in GAN Training}\label{sec:common}

\paragraph{Vanishing Gradients Issue}
In training GAN, discriminators might be too good for the generator to fool with and to improve progressively. When training with neural networks with back-propagation or gradient-based learning approaches, a vanishing small gradient only results in minor changes even with a large weight. As a result, the generator training may fail due to the vanishing gradients issue.

\paragraph{Mode Collapse Issue}
Mode collapse refers to the phenomenon that the generator will rotate through a small set of output types. For the given fixed discriminator, the generator over-optimizes in each iteration. Thus, the corresponding discriminator fails to learn its way out of the trap.

\paragraph{How \PG{} Alleviates the Vanish Gradient and Mode Collapse}
\PG{} alleviates the above two issues by preventing discriminators from "colluding" on its discrimination ability. In \PG{}, for either discriminator $D_i$, recall that $x_{p_1}$ and $x_{p_2}$ are randomly drawn from $\pd$ which are independent from each other. Then the max game of $D_i$, given its peer discriminator $D_j$, is to perform the following task:
\begin{align}
    \label{equ:punish}
    \max_{D_{i}}\mathcal{L}(D_{i},G)|_{D_{j}}=& \max_{D_i} \overbrace{\E_{x\thicksim p_{\text{data}}}[\log D_{i}(x)]+\E_{z\thicksim p_{z}}\Big[\log \Big(1-D_{i}\big(G(z)\big)\Big)\Big]}^{\text{Term \textcircled{a}}}\notag\\+&\beta \cdot\E_{x\thicksim \pd}\Big[\underbrace{\ell\Big(D_{i}(x),\BR\big(D_{j}(x)>\dfrac{1}{2}\big)\Big)}_\text{Term \textcircled{b}}\underbrace{-\alpha \cdot \ell\Big(D_{i}(x_{p_1}),\BR\big(D_{j}(x_{p_2})>\dfrac{1}{2}\big)\Big)}_\text{Term \textcircled{c}}\Big].
\end{align}
Term \textcircled{a} maximizes the probability of assigning the correct label to both real samples and generated samples. Term \textcircled{b} maximizes the probability of matching predicted label with peer discriminator predicted ones. In other words, Term \textcircled{b} controls the agreement level of $D_{i}$ with respect to its peer discriminator $D_j$. However, note that Term \textcircled{c} checks on the predictions of $D_{j}$ on two different tasks $x_{p_1}, x_{p_2}$. When $D_{i}$ agrees/fits overly on $D_{j}$, Term \textcircled{c} returns a lower value if $D_{j}$'s predictions on these two different tasks are matching, mathematically, $\BR\big(D_{j}(x_{p_1} )>\frac{1}{2}\big)= \BR(D_{j}\big(x_{p_2})>\frac{1}{2}\big)$. And Term \textcircled{c} will return a high value if $D_{j}$'s predictions on these two different tasks are indeed different $\BR\big(D_{j}(x_{p_1} )>\frac{1}{2}\big)\neq \BR\big(D_{j}(x_{p_2})>\frac{1}{2}\big)$. The weight $\alpha$ controls this disagreement level compared with Term \textcircled{b} by referring to the fact that a larger $\alpha$ encourages more disagreement/diverse predictions from discriminators.

Based on the above intuitions, when two discriminators are of a high disagreement level, there exists a set $S_{\text{dis}}$ such that $\BR(D_{i}(x)>\frac{1}{2})\neq \BR(D_{j}(x)>\frac{1}{2})$ for $x\in S_{\text{dis}}$ and $S_{\text{dis}}$ is non-negligible. Therefore, there exists at least one discriminator $D_{i}$ that can't perfectly predict labels (real/generated) of given data samples. The generator will then be provided with sufficient information, e.g., information or features that can be extracted from $S_{\text{dis}}$, to progress. This property helps us address the vanishing gradients issue. As for the mode collapse issue, suppose the over-optimized generator is able to find plausible outputs for both discriminators in the next generation. However, note that optimization is implemented on mini-batches in practice, the randomly selected samples $x_{p_1}, x_{p_2}$ in $\dd$ as well as the dynamically changing weights $\alpha, \beta$ can bring a certain degree of randomness in the next generation. Thus, rotating through this subset of the generator's output types could not force Term \textcircled{c} to remain unchanged, so that the discriminators won't maintain a constant disagreement level and they unlikely get stuck in a local optimum. In Section \ref{exp:mode_collapse}, we use synthetic experiments to show that \PG{} addresses mode collapse issues. And we include more empirical observations of the competition introduced by $\dd$ in the Appendix \ref{app:stab}, i.e., the stability of the \PG{} training, and the visualization of agreement levels between $D_1$ and $D_2$ due to the introduce of the duel game.

\subsection{Stability and Convergence Behavior}

In Section \ref{sec:common}, we discussed the significant role of the introduced intermediate duel game. Now we discuss the potential downsides of introducing a second discriminator. Particularly, we are interested in understanding if the introduce of a peer discriminator $D_j$ will disrupt the training and make the competition game with $D_i$ unstable. Suppose $D_{j}$ diverges from the optimum in the max game, in other words, the diverged peer discriminator $\tilde{D}_{j}$ fails to provide qualified verification label $Y^*_{j}$ (given by $D^*_{j,G}$), and provides $\tilde{Y}_{j}$ instead. Mathematically, denote:
\begin{align}
    e_{\text{data}, j}&:=\mathbb{P}(\tilde{Y}_{j}=0|Y^*_{j}=1), \quad
    e_{g, j}:=\mathbb{P}(\tilde{Y}_{j}=1|Y^*_{j}=0).
\end{align}
For any peer discriminator $D_j$, $D_j$ may be a diverged peer discriminator $\tilde{D}_j$ or an optimal one $D^*_{j,G}$, we denote the Duel Game of $D_i$ given her peer discriminator $D_j$ as:
    {\begin{align}
        \text{Duel}(D_i)|_{D_{j}}:=& \E_{x\thicksim \pd}\Big[\ell\Big(D_i(x),\BR\big(D_{j}(x)>\frac{1}{2}\big)\Big)-\alpha \cdot \ell\Big(D_i(x_{p_1}),\BR\big(D_{j}(x_{p_2})>\frac{1}{2}\big)\Big)\Big].
    \end{align}}
Theorem \ref{thm:stab1} explains the condition of stability (for $D_i$) when its peer discriminator in \PG{} diverges from the corresponding optimum.
\begin{theorem}
    \label{thm:stab1}
   Given $G$, suppose $D_{i}$ has enough capacity, and at one step of Algorithm 1, if $e_{data, j}+e_{g, j}<1$, $\alpha=1$, the duel term of discriminator $D_i$ is stable/robust with diverged peer discriminator $\tilde{D}_{j}$. Mathematically,
    \begin{align}
        \max_{D_{i}} \text{Duel}(D_i)|_{\tilde{D}_{j}}~~\text{is equivalent with}~~ \max_{D_{i}}\text{Duel}(D_i)|_{D_{j,G}^*}.
    \end{align}
\end{theorem}
The above theorem implies that a diverging and degrading peer discriminator $D_j$ will not disrupt the training of $D_i$.
\begin{remark}
    Note that assuming uniform prior of real and generated samples, the condition to be stable is merely requiring that the proportion of false/wrong $D_{j}$'s prediction is less than a half (random guessing). This condition can be easily satisfied in practice. Thus, Theorem \ref{thm:stab1} provides the stability/robustness guarantee when the peer discriminator diverged from its optimum.
\end{remark}

Build upon Theorem \ref{thm: global_min}, with sufficiently small updates, Theorem \ref{thm: conv1} presents when $p_g$ converges to $p_{\text{data}}$.
\begin{theorem}
\label{thm: conv1}
    If $G$ and $D_i$s have enough capacity, and at each step of Algorithm 1, $D_i$s are allowed to reach its optimum given $G$, $D_i$ is updated so as to improve the criterion in Eqn.(\ref{equ:punish}), and $p_g$ is updated so as to improve:
    \begin{align}
   C(G)=&\E_{x\thicksim p_{\text{data}_1}}[\log D_{1, G}^{*}(x)]+\E_{x\thicksim p_{g_1}}\big[\log \big(1-D_{1, G}^{*}(x)\big)\big] \notag\\
     +&\E_{x\thicksim p_{\text{data}_2}}[\log D_{2, G}^{*}(x)]+\E_{x\thicksim p_{g_2}}\big[\log \big(1-D_{2, G}^{*}(x)\big)\big].
    % \end{split}
\end{align}
If $\beta=0$, we have $D_{1,G}^*=D_{2,G}^*$, $p_{g}$ converges to $p_{\text{data}}$.
\end{theorem}

\section{Experiments}
In this section, we empirically validate the properties of \PG{} through a set of datasets, including a synthetic task and several real world datasets ranging from hand-written digits to human faces.

\subsection{Experiment Results on Synthetic Data}
\label{exp:mode_collapse}

\begin{figure*}[htb]
  \centering

\includegraphics[width=0.3\textwidth]{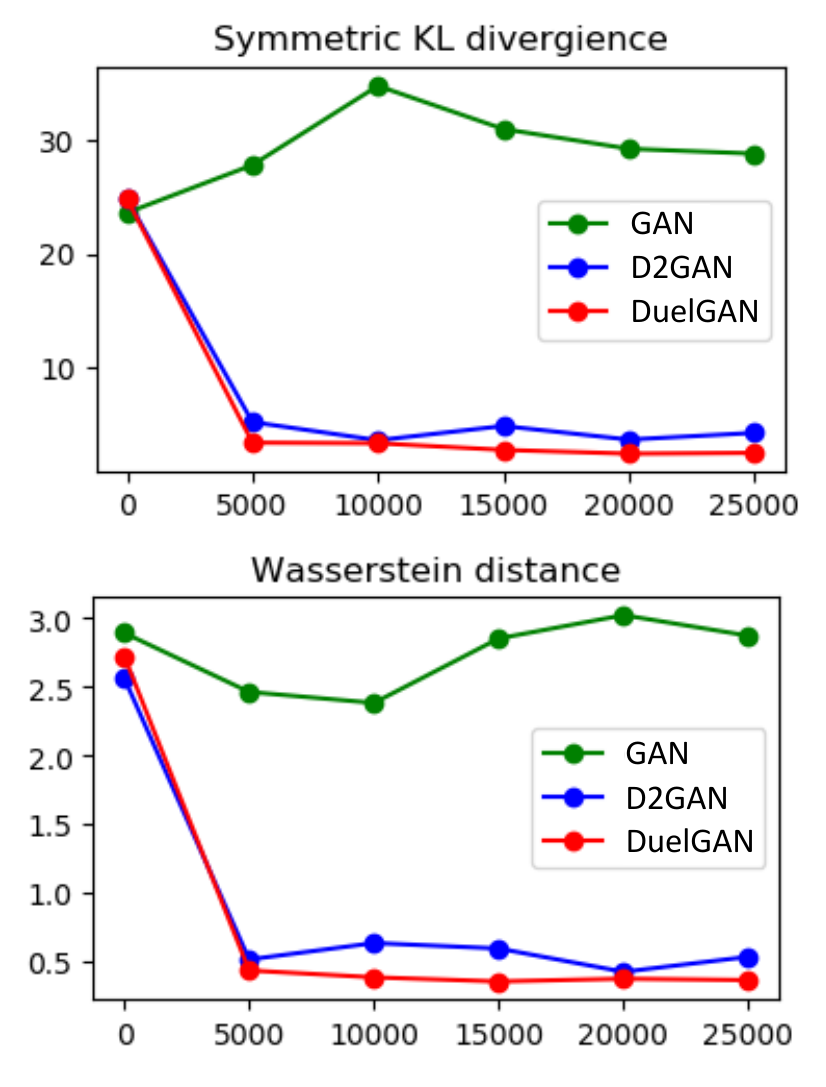}
\includegraphics[width=0.64\textwidth]{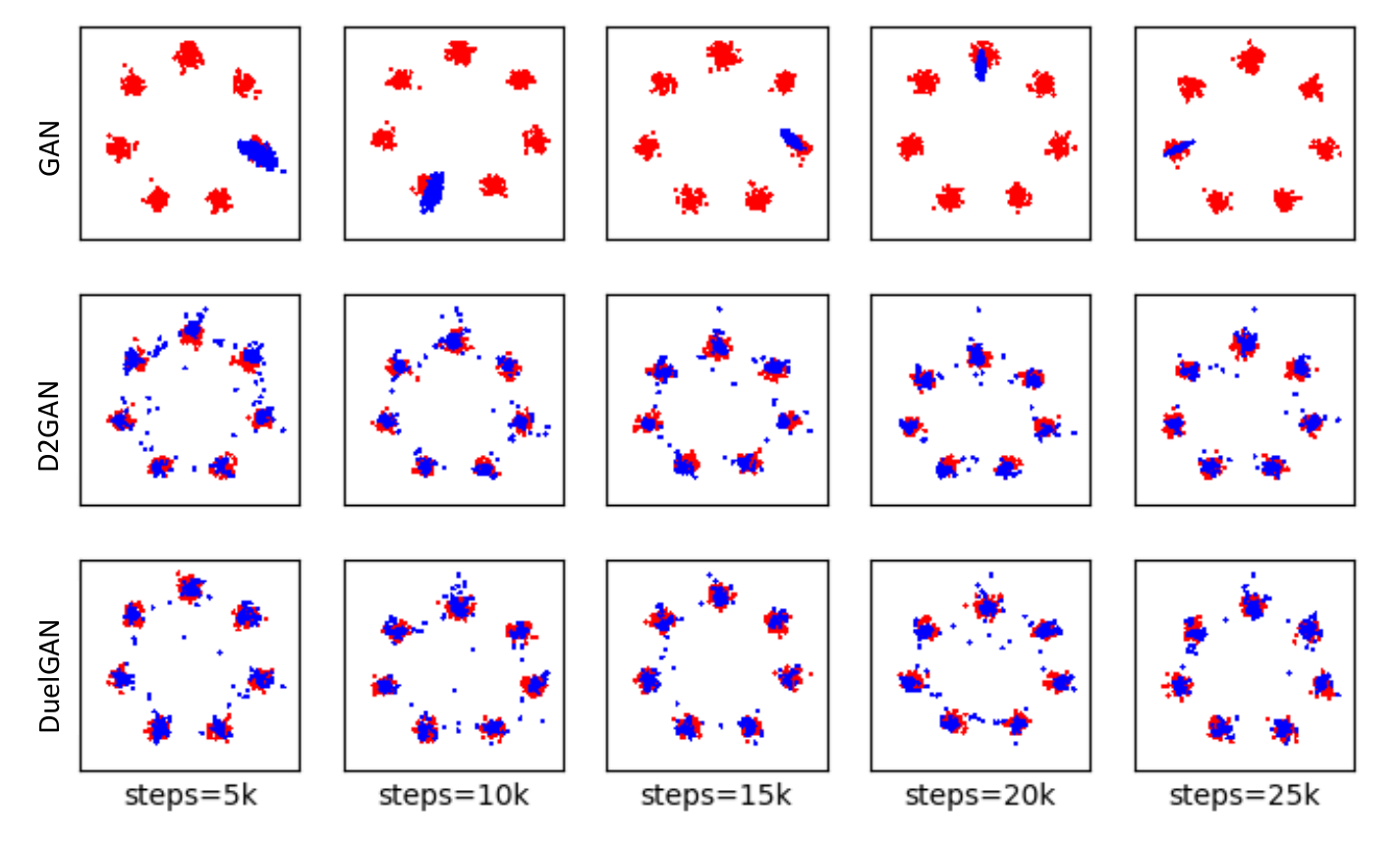}
% \vspace{-10pt}
\caption{Comparison of Vanilla GAN, D2GAN, and proposed \PG{} on 2D synthesized data. The top-left graph shows the symmetric KL divergence over the training iterations, while the bottom left graph shows the Wasserstein distance. Both metrics compare the generated data points to data points drawn from the true target distribution. \PG{} has the best performance. The right side visualizes generated {\color{blue}blue} data points and true {\color{red}red} $p_{\text{data}}$ data points. Note that Vanilla GAN has a clear mode collapse which both D2GAN and \PG{} avoid.}\label{synthesized data results}
\end{figure*}

We apply the experiment and model structures proposed in UnrolledGAN \cite{UnrolledGAN} to investigate whether the \PG{} design can prevent mode collapse. This experiment aims to generate eight 2D Gaussian distributions with a covariance matrix $0.02I$,  arranged around the same centroid with radius 2.0. Vanilla GAN fails on this example. D2GAN has been shown to outperform UnrolledGAN, so we include it as an alternate method which performs well. 

Figure \ref{synthesized data results} shows symmetric KL-divergence, Wasserstein distance, and a visualization of results with Vanilla GAN, D2GAN, and \PG{}. Knowing the target distribution $p_{\text{data}}$, we can employ symmetric KL divergence and Wasserstein distance, which calculate the distance between the true $p_{\text{data}}$ and the normalized histogram of 10,000 generated points. On the left of Figure \ref{synthesized data results}, the plots for symmetric KL-divergence and Wasserstein distance show that \PG{} has a much better score than Vanilla GAN and slightly better than D2GAN. 

On the right side of Figure \ref{synthesized data results} is a visualization of 512 generated blue samples points, together with red data points drawn from the true distribution. Vanilla GAN generates data points around only a single valid mode of the data distribution. D2GAN and \PG{} distribute data around all eight mixture components, demonstrating the ability to resolve modal collapse in this case.

\subsection{Experiments on Real Image Datasets}
We tested the proposed \PG{} and baseline methods on MNIST \cite{MNIST}, FashionMNIST~\cite{FashionMNST}, CIFAR-10~\cite{Cifar-10}, STL-10~\cite{STL-10}, CelebA~\cite{CelebA} and VGGFace2~\cite{VGGFace2}. For quantitative evaluation, we adopt Fr\'echet Inception Distance (FID) \cite{FID} and Inception score(IS) \cite{IS} as the evaluation metric. FID summarizes the distance between the Inception features of the generated images and the real images.  A lower FID indicates both better accuracy and higher diversity, so that a batch of generated images with good accuracy but identical to each other will have a poor FID score. A higher IS score indicates a higher generated image quality. 

\paragraph{Baseline Methods}

We reproduce/report the performance of a list of existing baseline methods, including: DCGAN \cite{DCGAN}, D2GAN \cite{D2GAN}, WGAN \cite{iwgan}, DRAGAN \cite{dragan}, LSGAN \cite{lsgan}, MicroBatchGAN \cite{microbatchgan}, Dist-GAN \cite{dist-gan}, PresGAN \cite{dieng2019prescribed}, and QSNGAN \cite{grassucci2021quaternion}. We used the same generator and discriminator backbone for all the comparison methods in each dataset unless specified by the original author. We recorded the best performing checkpoints when evaluating each method.

\paragraph{Grey-Scale Images}
MNIST~\cite{MNIST} and FashionMNIST~\cite{FashionMNST} are small grey-scale image datasets including 60,000 training and 10,000 testing 28$\times$28 gray-scale images of hand-written digits and clothing. Since they are of small-scale, we adopt the shallow version of the generator and discriminators to generate the grey-scale images. We firstly give the performance comparisons between \PG{} and baseline methods that only adopted the Inception score in the original paper. We then include a comprehensive comparison via FID score in Table \ref{tab: Fid result}. And the first two columns in Table \ref{tab: Fid result} show our method has the best FID score among all tested methods. Figure \ref{generated images} (left) shows FashionMNIST image results.

\begin{table}[H]
\centering
\caption{Inception score results of CIFAR-10 and STL-10.}\label{tab: IS result}
 \begin{tabular}{l|cc}
\hline
              &CIFAR10   &  STL-10 \\ \hline
WGAN   &3.82 &3.97 \\
GAN	   & 2.61	& 2.17 \\
MicroBatchGAN  &6.77 &\textbf{7.23} \\
DCGAN    & 6.40 &   5.87\\ 
D2GAN     & 7.15 &  6.15\\ 
\textbf{\PG{} (ours)}            &\textbf{7.45}  & 6.22\\ \hline
\end{tabular}  
\end{table} 
{
\setlength{\tabcolsep}{2pt}
\begin{table}[H]
\centering
\caption{ Experiment FID score results of grey-scale image dataset: MNIST and FashionMNIST; natural scene image dataset: CIFAR-10 and STL-10; human face image dataset: CelebA and VGGFace2. Baseline results denoted with (*) were extracted from the original paper report, not independently run in our experiments.}
{
\begin{tabular}{l|cccccc}
\hline
            & MNIST    & FasionMNIST   &CIFAR10   & \:\:\:\:STL-10\:\:\:\:\: &\:\:CelebA\:\:   &\:\:\:\:\:\:VGG\:\:\:\:\:\:\\ \hline

DCGAN \cite{DCGAN}      & 19.86 & 24.78 & 27.45 & 59.79 & 17.38 &   49.99\\ 
WGAN* \cite{iwgan}	& 14.07	& 28.24	 & 35.37	& 60.21 & 15.23	& 39.24 \\
% WGAN-GP \cite{gulrajani2017improved} & -- & -- & 29.30 & -- & -- & --\\
DRAGAN \cite{dragan}	& 66.96	& 62.64	& 36.49	& 91.07	& 14.57	& 50.20\\
D2GAN \cite{D2GAN}      & 22.20 & 29.33  & 27.38 & 54.12  &17.30 &  20.67\\ 
Dist-GAN* \cite{dist-gan} &-- &-- &22.95 &\textbf{36.19} &23.7 &-- \\
PresGAN* \cite{dieng2019prescribed} & 42.02 & -- & 52.20 & -- & -- & -- \\
LSGAN \cite{lsgan}	& 23.80 	& 43.00 & 51.42	& 70.37 & 15.35	& 55.96 \\
MicroBatchGAN* \cite{microbatchgan} & 17.10 & --	& 77.70	& -- & 34.50 & -- \\
QSNGAN* \cite{grassucci2021quaternion} & -- & -- & 31.97  & 59.61 & -- & -- \\
\textbf{\PG{} (ours)}    & \textbf{\ \  7.87} & \textbf{21.73} & \textbf{21.55} &51.37 &\textbf{13.95}  & \textbf{19.05}\\ \hline
\end{tabular}
\label{tab: Fid result}
% }
}
\end{table}
}
\paragraph{Natural Scene Images}
CIFAR-10~\cite{Cifar-10} and STL-10~\cite{STL-10} are natural scene RGB image datasets. CIFAR-10 includes 50,000 training and 10,000 testing 32$\times$32  images with ten unique categories: airplane, automobile, bird, cat, deer, dog, frog, horse, ship, and truck. STL-10 is sub-sampled from ImageNet, and has more diverse samples than CIFAR-10, containing about 100,000 96$\times$96 images. We adopt the deep version of the generator and discriminator to generate 32$\times$32 RGB images. Table \ref{tab: Fid result} middle two columns show FID score results and Table \ref{tab: IS result} shows the inception score results. Note that the introduce of competitive Duel Game in two discriminator GAN setup, brings performance boost in all the experiments. Figure \ref{generated images} (middle) shows STL-10 image results.

\begin{figure*}[!htb]
  \centering
\includegraphics[width=0.32\textwidth]{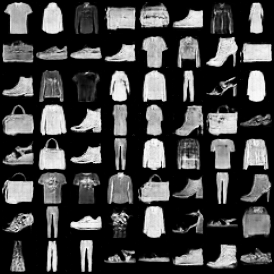}
\includegraphics[width=0.32\textwidth]{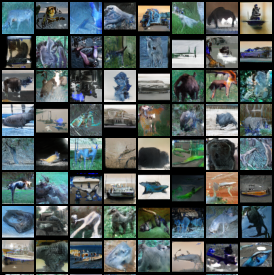}
\includegraphics[width=0.32\textwidth]{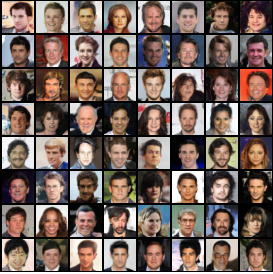}
\caption{Image results generated by proposed \PG{}. Left: FashionMNIST, grey-scale clothing images; Middle: STL-10, natural scene images; Right: CelebA, large-scale celebrate face images.}\label{generated images}
\end{figure*}
\paragraph{Human Face Images}
CelebA~\cite{CelebA} and VGGFace2~\cite{VGGFace2} are large-scale face datasets. CelebA includes 162,770 training and 19,962 testing images of celebrity faces. VGGFace2 contains more than 3.3 million face images of celebrities caught in the `wild'. There are different lighting conditions, emotions, and viewing angles. We randomly choose 200 categories from VGGFace2 and trained on the reduced dataset. We adopt the deep version of the generator and discriminators to generate 32$\times$32 RGB images on CelebA and 64$\times$64 RGB images on VGGFace2. Table \ref{tab: Fid result} last two columns show our method has the best FID score among tested methods. Figure \ref{generated images} (right) shows CelebA image results.

\paragraph{Implementation Details} 
Our model architecture adopts the same generator and discriminator backbone as DCGAN~\cite{DCGAN}.  In \PG{}, the newly introduced discriminator is a duplicate of the first one. \PG{} achieves low FID scores and high IS scores when $\alpha$ and $\beta$ are simply set to constant values. However we found that we could obtain an approximately 10\% improvement through dynamic tuning. The parameter $\beta$ controls the overall weight of $\dd$, while $\alpha$ punishes the condition when $D_1$ over-agrees with $D_2$. In the early training phase, when we have an unstable generator and discriminator, we set $\alpha$ and $\beta$ to 0. As training progresses, we gradually increase these parameters to a max value, which helps with vanishing gradients. After the midpoint of training we decrease these parameters to help the discriminators converge, until the parameters reach approximately 0 at the end of the training process. We adopt 0.3, 0.5 as the max value for $\alpha$ and $\beta$, respectively.

\subsection{Duel Game as a Regularizer}
Intuitively, the introduced duel game could be well applied to a large family of GAN variants defined w.r.t a single discriminator $D_1$ and a generator $G$. This is due to the fact that Eqn.(\ref{equ:peergan}) could be denoted by:
\begin{tcolorbox}[colback=grey!10!white,colframe=grey!10!white]
\begin{align}
    &\min_{G} \max_{D_1, D_2}\mathcal{L}(D_1, D_2, G)=\min_{G} \max_{D_1, D_2}\left[ \text{GAN}(D_1)+\beta \cdot \Cline{\dd}+\text{GAN}(D_2)\right],
\end{align}
\end{tcolorbox}
where $\text{GAN}(D_i):=\E_{x\thicksim p_{\text{data}}}\left[\log D_i(x)\right]+\E_{z\thicksim p_{z}}\left[\log \left(1-D_i\left(G(z)\right)\right)\right]$. Thus, if we substitute the GAN loss $\text{GAN}(D_i)$ by a state-of-the-art GAN variant, i.e., StyleGAN-ADA \cite{karras2020training}, one could view the duel game $\dd$ as a regularizer. 

We take the higher resolution version (256$\times$256 RGB images) of CelebA \cite{CelebA} for illustration. Clearly in Table \ref{tab:celeba_high}, StyleGAN-ADA reaches the state-of-the-art result on this task. And the introduced $\dd$ regularizer could further improve its performance. Figure \ref{fig:celeba_high} shows the corresponding generated images.
 
\begin{table*}[t]
\centering
\caption{Experiment FID score results of CelebA (256$\times$256 RGB images). Baseline results denoted with (*) were obtained from the original paper report.}
{\scalebox{0.93}{\begin{tabular}{c|c|c|c|c|c|c}
\hline
     Method & GLF* \cite{xiao2019generative}  & MSP* \cite{li2020latent}  & NCP-VAE* \cite{aneja2021contrastive}  & LSGM* \cite{vahdat2021score} & StyleGAN-ADA \cite{karras2020training} & StyleGAN-ADA+$\dd$  \\ \hline
    FID & 41.80  &  35.00 & 24.79  & 7.22 &  4.85& \textbf{4.32} \\ \hline
\end{tabular}}}\label{tab:celeba_high}
\end{table*}

\begin{figure*}[!htb]
  \centering
\includegraphics[width=\textwidth,trim={0 27.3cm 0 0},clip]{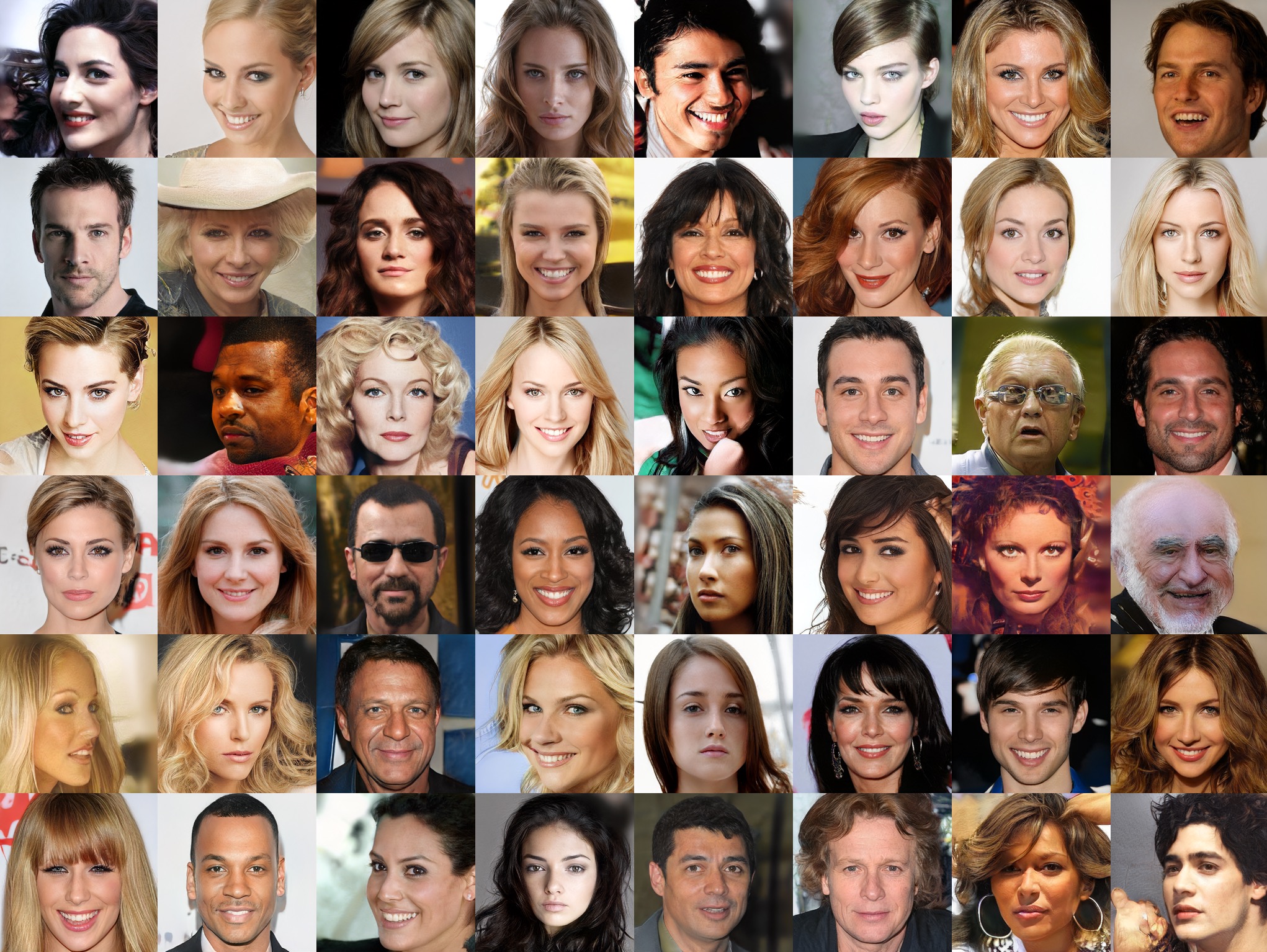} 
\caption{Image results generated by proposed \PG{}. (Trained on CelebA 256$\times$256 RGB images. More generated images are deferred to the Appendix \ref{app:exp}.)}\label{fig:celeba_high}
\end{figure*}

\paragraph{More Experiment Results} 
We defer more experiment results to the Appendix \ref{app:exp}, including: an ablation study of hyper-parameters tuning; experiment validations about the stability of training; the visualization of the duel game between $D_1$ and $D_2$.

\section{Conclusion}
We propose \PG{} which introduces a peer discriminator to Vanilla GAN. The role of the peer discriminator is to allow an intermediate game (duel game) between discriminators. Theoretical analysis demonstrates that the introduced duel game incentivizes incremental improvement, addresses vanishing gradients and mode collapse issues, punishes over-agreements among discriminators and is stable with diverged peer discriminator. Experimental results on a synthetic dataset and multiple real world datasets validate that \PG{} produces high quality images, with lower error than competing techniques.

\clearpage
\newpage
\bibliographystyle{plain}
\bibliography{gan, references, exp}

\newpage
\onecolumn
\appendix

\newpage
\begin{center}
    \section*{\Large Appendix}
\end{center}
The appendix is organized as follows:
\squishlist
    \item Section A includes the omitted proofs for all theoretical conclusions in the main paper.
    \item Section B includes experiment details and additional experiment results.
\squishend

\section{Omitted Proofs}

\subsection{Proof of Proposition \ref{prop: opt_d}}\label{app:p1}
We firstly introduce Lemma 1 which helps with the proof of Proposition \ref{prop: opt_d}.

\begin{lemma}\label{lm:basic}
For any $(a,b)\in \mathbb{R}^2\setminus \{0, 0\}$, the function $y\rightarrow a\log(y)+b\log(1-y)$ achieves its maximum in $[0,1]$ at $\frac{a}{a+b}$.
\end{lemma}
\begin{proof}
Denote by $f(y):=a\log(y)+b\log(1-y)$, clearly, when $y=0$ or $y=1$, $f(y)=-\infty$. For $y\in (0,1)$, we have:
\begin{align}
    f'(y)=0\Longleftrightarrow \frac{a}{y}-\frac{b}{1-y}=0 \Longleftrightarrow y=\frac{a}{a+b}.
\end{align}
Note that $f'(y)>0$ if $0<y<\frac{a}{a+b}$ and $f'(y)<0$ if $1>y>\frac{a}{a+b}$. Thus, the maximum of $f(y)$ should be $\max(f(a), f(\frac{a}{a+b}), f(b))=f(\frac{a}{a+b})$. And $f(y)$ achieves its maximum in $[0,1]$ at $\frac{a}{a+b}$.
\end{proof}
Now we proceed to prove Proposition \ref{prop: opt_d}. 
\paragraph{\textbf{Proof of Proposition \ref{prop: opt_d}}}
\begin{proof}
The trainer criterion for the discriminator $D_i$, given any generator $G$, is to maximize the quantity $\mathcal{L}(D_1, D_2,G)$. Remember that:
\begin{align}\label{eqn:19}
    \mathcal{L}(D_1, D_2, G)
    &= \E_{x\thicksim p_{\text{data}}}\left[\log D_1(x)\right]+\E_{x\thicksim p_{\text{data}}}\left[\log D_2(x)\right]\notag\\
    &+\E_{z\thicksim p_{z}}\left[\log \left(1-D_1\left(G(z)\right)\right)\right]+\E_{z\thicksim p_{z}}\left[\log \left(1-D_2\left(G(z)\right)\right)\right]\notag\\
    &+\beta \cdot \E_{x\thicksim \pd}\Bigg[\ell\Big(D_1(x),\BR\big(D_2(x)>\dfrac{1}{2}\big)\Big)-\alpha \cdot \ell\Big(D_1(x_{p_1}),\BR\big(D_2(x_{p_2})>\dfrac{1}{2}\big)\Big)\Bigg]\notag\\
     &+\beta \cdot \E_{x\thicksim \pd}\Bigg[\ell\Big(D_2(x),\BR\big(D_1(x)>\dfrac{1}{2}\big)\Big)-\alpha \cdot \ell\Big(D_2(x_{p_1}),\BR\big(D_1(x_{p_2})>\dfrac{1}{2}\big)\Big)\Bigg].
\end{align}
We then have:
{
\begin{align}
    \begin{split}
    \text{Eqn.}(\ref{eqn:19})&= \int_{x}p_{\text{data}}(x)\big[\log\big(D_1(x)\big)+\log\big(D_2(x)\big)\big]dx+\int_{z}p_{z}(z)\Big[\log\Big(1-D_1\big(G(z)\big)\Big)+\log\Big(1-D_2\big(G(z)\big)\Big)\Big]dz\\
    &+\beta \cdot \int_{x} \pd(x)\big(r_{2,G}(x)-\alpha \cdot p_{2,G}\big)\cdot \log\big(D_1(x)\big)dx+\beta \cdot \int_{x} \pd(x)\big(r_{1,G}(x)-\alpha \cdot p_{1,G}\big)\cdot \log\big(D_2(x)\big)\big]dx\\
    &+\beta\cdot \int_{x} \pd(x)\big(1-\alpha-r_{2,G}(x)+\alpha \cdot p_{2,G}\big)\cdot \log\big(1-D_1(x)\big)dx\\
    &+\beta \cdot \int_{x} \pd(x)\big(1-\alpha-r_{1,G}(x)+\alpha \cdot p_{1,G}\big)\cdot \log\big(1-D_2(x)\big)dx\\
    &= \int_{x}p_{\text{data}}(x)\big[\log\big(D_1(x)\big)+\log\big(D_2(x)\big)\big]dx+\int_{x}p_{g}(x)\big[\log\big(1-D_1(x)\big)+\log\big(1-D_2(x)\big)\big]dx\\
    &+\beta \cdot \int_{x} \pd(x)\big(r_{2,G}(x)-\alpha \cdot p_{2,G}\big)\cdot \log\big(D_1(x)\big)dx+\beta \cdot \int_{x} \pd(x)\big(r_{1,G}(x)-\alpha \cdot p_{1,G}\big)\cdot \log\big(D_2(x)\big)dx\\
    &+\beta\cdot \int_{x} \pd(x)\big(1-\alpha-r_{2,G}(x)+\alpha \cdot p_{2,G}\big)\cdot \log\big(1-D_1(x)\big)dx\\
    &+\beta \cdot \int_{x} \pd(x)\big(1-\alpha-r_{1,G}(x)+\alpha \cdot p_{1,G}\big)\cdot \log\big(1-D_2(x)\big)dx\\
    &= \int_{x} \big[p_{\text{data}}(x)+\beta \cdot \big(r_{2,G}(x)-\alpha \cdot p_{2,G}\big)\cdot \pd(x)\big] \cdot\log\big(D_1(x)\big) dx\\
    &+ \int_{x} \big[p_{g}(x)+\beta \cdot \big(1-\alpha-r_{2,G}(x)+\alpha \cdot p_{2,G}\big)\cdot \pd(x)\big] \cdot\log(1-D_1(x)) dx\\
    &+ \int_{x} \big[p_{\text{data}}(x)+\beta \cdot \big(r_{1,G}(x)-\alpha \cdot p_{1,G}\big)\cdot \pd(x)\big] \cdot\log\big(D_2(x)\big) dx\\
    &+ \int_{x} \big[p_{g}(x)+\beta \cdot \big(1-\alpha-r_{1,G}(x)+\alpha \cdot p_{1,G}\big)\cdot \pd(x)\big] \cdot\log\big(1-D_2(x)\big) dx.\\
    \end{split}
\end{align}}
For $D_1, D_2$, according to Lemma \ref{lm:basic}, the above objective function respectively achieves its maximum in $[0, 1], [0, 1]$ at: 
\begin{align}
D_{i,G}^*(x)=\dfrac{p_{\text{data}}(x)+\beta \cdot (r_{j,G}(x)-\alpha \cdot p_{j, G})\cdot \pd(x)}{p_{\text{data}}(x)+p_{g}(x)+\beta \cdot (1-\alpha)\cdot \pd(x)}, \qquad i\neq j.
\end{align}
With the introduce of Duel Game, the distributions $p_{\text{data}}$ and $p_{g}$ in the Vanilla GAN got changed due to the appearance of $\pd$. Thus, we define the corresponding updated distributions in \PG{} w.r.t. discriminator $D_i$ as $p_{\text{data}_i}$ and $p_{g_i}$, respectively:
\begin{align}
\label{eqn:new_dis}
    &p_{\text{data}_i}(x) :=\dfrac{p_{\text{data}}(x)+\beta\cdot \hat{r}^*_{j,G}(x)\cdot \pd(x)}{\int_{x}p_{\text{data}}(x)+\beta\cdot \hat{r}^*_{j,G}(x)\cdot \pd(x)dx},\\
    &p_{g_i}(x) :=\dfrac{p_{g}(x)+\beta\cdot \big(1-\hat{r}^*_{j,G}(x)\big)\cdot \pd(x)}{\int_{x}p_{g}(x)+\beta\cdot \big(1-\hat{r}^*_{j,G}(x)\big)\cdot \pd(x)dx}.
\end{align}
\end{proof}

\subsection{Proof of Theorem \ref{thm: global_min}}

\begin{proof}
When $\alpha=0, r_{j,G}(x)=\frac{1}{2}$, for $\alpha= 0, i=1,2$, we have:
\begin{align}
    D_{i,G}^*(x)&=\dfrac{p_{\text{data}}(x)+\beta \cdot \hat{r}_{j,G}^*(x)\cdot \pd(x)}{p_{\text{data}}(x)+p_{g}(x)+\beta\cdot \pd(x)}\rightarrow \dfrac{p_{\text{data}}(x)+\dfrac{\beta}{2} \cdot \pd(x)}{p_{\text{data}}(x)+p_{g}(x)+\beta\cdot \pd(x)}.
\end{align}
This allows us to rewrite $\frac{C(G)}{2}$ as:
\begin{align}
    \frac{C(G)}{2}=&\E_{x\thicksim p_{\text{data}_i}}\left[\log \frac{p_{\text{data}}(x)+\frac{\beta}{2}\cdot \pd(x)}{p_{\text{data}}(x)+p_{g}(x)+\beta\cdot \pd(x)}\right]+\E_{x\thicksim p_{g_i
     }}\left[\log \frac{p_{g}(x)+\frac{\beta}{2}\cdot \pd(x)}{p_{\text{data}}(x)+p_{g}(x)+\beta\cdot \pd(x)}\right].
\end{align}
$\Longrightarrow$
Note that $2\cdot \big(\E_{x\thicksim p_{\text{data}_i}}[-\log2]+\E_{x\thicksim p_{g_i}}[-\log2]\big)=-\log{16}$, by subtracting this expression from $C(G)$, we have:
\begin{align}
    C(G)=&-\log{16}+ 2\cdot KL\Big(p_{g}+ \dfrac{\beta}{2}\cdot \pd\Big|\Big|\dfrac{p_{\text{data}}+p_{g}+\beta\cdot \pd}{2}\Big)\notag\\
    &+2\cdot KL\Big(p_{\text{data}}+\dfrac{\beta}{2}\cdot \pd\Big|\Big|\dfrac{p_{\text{data}}+p_{g}+\beta\cdot \pd}{2}\Big),
\end{align}
where KL is the Kullback-Leibler divergence. Note that:
\begin{align}
    C(G)=-\log{16}+2\cdot  JSD\Big(p_{\text{data}}+\dfrac{\beta}{2}\cdot \pd\Big|\Big|p_{g}+\dfrac{\beta}{2}\cdot \pd\Big),
\end{align}
and the Jensen-Shannon divergence between two distributions is always non-negative and zero only when they are equal, we have shown that $C(G)^*=-\log{16}$ is the global minimum of $C(G)$. Thus, we need $$p_{\text{data}}+\dfrac{\beta}{2}\cdot \pd=p_{g}+\dfrac{\beta}{2}\cdot \pd\Leftrightarrow p_{\text{data}}=p_g.$$
$\Longleftarrow$
Given that $p_{\text{data}}=p_{g}$, we have:
\begin{align}
    C(G)=&\max_{D}\mathcal{L}(G, D_1, D_2)\notag\\
     =&2\cdot \E_{x\thicksim p_{\text{data}_i}}\left[\log \dfrac{p_{\text{data}}(x)+\dfrac{\beta}{2}\cdot \pd(x)}{p_{\text{data}}(x)+p_{g}(x)+\beta\cdot \pd(x)}\right]+2\cdot \E_{x\thicksim p_{g_i}}\left[\log \dfrac{p_{g}(x)+\dfrac{\beta}{2}\cdot \pd(x)}{p_{\text{data}}(x)+p_{g}(x)+\beta\cdot \pd(x)}\right]\notag\\
     =& 2\cdot \left(\log\dfrac{1}{2}+\log\dfrac{1}{2}\right)=-\log{16}.
\end{align}

\end{proof}

\subsection{Proof of Theorem \ref{thm:stab1}}
\begin{proof}
Ignoring the weight $\beta$, the duel term of discriminator $D_i$ w.r.t. its diverged peer discriminator $\tilde{D}_{j}$ becomes:
{\begin{align}\label{eqn:split}
    &\quad\text{Duel}(D_i)|_{\tilde{D}_{j}}:= \E_{x\thicksim \pd}\Big[\ell\Big(D_i(x),\BR\big(\tilde{D}_{j}(x)>\dfrac{1}{2}\big)\Big)-\alpha \cdot \ell\Big(D_i(x_{p_1}),\BR\big(\tilde{D}_{j}(x_{p_2})>\dfrac{1}{2}\big)\Big)\Big]\notag\\
    &=\E_{x\thicksim \pd,Y^*_j=1}\Big[\mathbb{P}(\tilde{Y}_j=1|Y^*_j=1)\cdot \ell\big(D_{i}(x),1\big)+\mathbb{P}(\tilde{Y}_j=0|Y^*_j=1)\cdot \ell\big(D_{i}(x),0\big)\Big]\notag
    \\
    &+\E_{x\thicksim \pd,Y^*_j=0}\Big[\mathbb{P}(\tilde{Y}_j=1|Y^*_j=0)\cdot \ell\big(D_{i}(x),1\big)+\mathbb{P}(\tilde{Y}_j=0|Y^*_j=0)\cdot \ell\big(D_{i}(x),0\big)\Big]\notag\\
    &-\alpha \cdot \E_{x_{p_1}\thicksim \pd}\Big[\mathbb{P}(\tilde{Y}_{j}=1)\cdot \ell\big(D_{i}(x_{p_1}),1\big)+\mathbb{P}(\tilde{Y}_{j}=0)\cdot \ell\big(D_{i}(x_{p_1}),0\big)\Big]\notag\\
    &=\E_{x\thicksim \pd,Y^*_j=1}\Big[(1-e_{\text{data}, j})\cdot \ell(D_{i}(x),1)+e_{\text{data}, j}\cdot \ell(D_{i}(x),0)\Big]
    \notag\\
    &+\E_{x\thicksim \pd,Y^*_j=0}\Big[e_{g,j}\cdot \ell(D_{i}(x),1)+(1-e_{g,j})\cdot \ell(D_{i}(x),0)\Big]\notag\\
    &-\alpha \cdot \E_{x_{p_1}\thicksim \pd}\Big[\big[\mathbb{P}(Y^*_j=1)\cdot (1-e_{\text{data},j})+\mathbb{P}(Y^*_j=0)\cdot e_{g,j}\big]\cdot \ell\big(D_{i}(x_{p_1}),1\big)\Big]\notag\\
    &-\alpha \cdot \E_{x_{p_1}\thicksim \pd}\Big[\big[\mathbb{P}(Y^*_j=1)\cdot e_{\text{data},j}+\mathbb{P}(Y^*_j=0)\cdot (1-e_{g,j})\big]\cdot \ell\big(D_{i}(x_{p_1}),0\big)\Big]\notag\\
    &=\E_{x\thicksim \pd,Y_j^*=1}\Big[(1-e_{\text{data}, j}-e_{g, j})\cdot \ell(D_{i}(x),1)+e_{\text{data}, j}\cdot \ell(D_{i}(x),0)+e_{g, j}\cdot \ell(D_{i}(x),1)\Big]
   \notag \\
    &+\E_{x\thicksim \pd,Y_j^*=0}\Big[(1-e_{\text{data}, j}-e_{g,j})\cdot \ell(D_{i}(x),0)+e_{data,j}\cdot \ell(D_{i}(x),0)+e_{g,j}\cdot \ell(D_{i}(x),1)\Big]\notag\\
      &-\alpha \cdot \E_{x_{p_1}\thicksim \pd}\Big[c_1\cdot \ell\big(D_{i}(x_{p_1}),1\big)\Big]\notag-\alpha \cdot \E_{x_{p_1}\thicksim \pd}\Big[c_2\cdot \ell\big(D_{i}(x_{p_1}),0\big)\Big],
\end{align}}
where we define: 
\begin{align*}
  &c_1:=\mathbb{P}(Y^*_j=1)\cdot (1-e_{\text{data},j}-e_{g,j})+\mathbb{P}(Y^*_j=0)\cdot e_{g,j}+\mathbb{P}(Y^*_j=1)\cdot e_{g,j},\\
  &c_2:=\mathbb{P}(Y^*_j=0)\cdot (1-e_{\text{data},j}-e_{g,j})+\mathbb{P}(Y^*_j=1)\cdot e_{\text{data},j}+\mathbb{P}(Y^*_j=0)\cdot e_{\text{data},j},
\end{align*}
for a clear presentation. Proceeding the previous deduction, we then have:
\begin{align}
    \text{Duel}(D_i)|_{\tilde{D}_{j}}&=(1-e_{\text{data}, j}-e_{g, j})\cdot\E_{x\thicksim \pd}\Big[ \ell\big(D_{i}(x),Y^*_j\big)\Big]+\E_{x\thicksim \pd}\Big[e_{\text{data}, j}\cdot \ell\big(D_{i}(x),0\big)+e_{g, j}\cdot \ell\big(D_{i}(x),1\big)\Big]\notag
    \\
    &-\alpha \cdot (1-e_{\text{data},j}-e_{g,j})\cdot  \E_{x\thicksim \pd}\Big[ \ell\big(D_{i}(x_{p_1}),Y^*_j\big)\Big]-\alpha \cdot \E_{x\thicksim \pd}\Big[e_{\text{data}, j}\cdot \ell\big(D_{i}(x),0\big)+e_{g, j}\cdot \ell\big(D_{i}(x),1\big)\Big].
\end{align}
Thus, 
\begin{align}
    \text{Duel}(D_i)|_{\tilde{D}_{j}}=&(1-e_{\text{data}, j}-e_{g, j})\cdot \text{Duel}(D_i)|_{D^*_{j,G}}\notag\\
    +&\underbrace{(1-\alpha)\cdot \E_{x\thicksim \pd} \big[e_{\text{data}, j}\cdot \ell\big(D_{i}(x),0\big)+e_{g, j}\cdot \ell\big(D_{i}(x),1\big)\big]}_{\textbf{Bias}}.
\end{align}
Note that:
\begin{align}
    \textbf{Bias}=(1-\alpha)\cdot \E_{x\thicksim \pd} \big[e_{\text{data}, j}\cdot \log\big(1-D_{i}(x)\big)+e_{g, j}\cdot \log\big(D_{i}(x)\big)\big].
\end{align}
Thus, given $\alpha=1$, the \textbf{Bias} term is cancelled out. When $e_{\text{data}, j}+e_{g, j}<1$, we have:
\begin{align}
    \text{Duel}(D_i)|_{\tilde{D}_{j}}=&(1-e_{\text{data}, j}-e_{g, j})\cdot \text{Duel}(D_i)|_{D^*_{j,G}},
\end{align}
and we further have:
\begin{align}
    \max_{D_{i}} \text{Duel}(D_i)|_{\tilde{D}_{j}}
    =& \max_{D_{i}} \text{Duel}(D_i)|_{D^*_{j,G}}.
\end{align}

\end{proof}

\subsection{Proof of Theorem \ref{thm: conv1}}
\begin{proof}
When $\beta=0$, the overall min-max game becomes:
\begin{align}
    &\min_{G}\max_{D_1,D_2}\mathcal{L}(D_1,D_2, G)\notag\\
    =& \min_{G} \max_{D_1,D_2} \E_{x\thicksim p_{\text{data}}}\big[\log D_1(x)\big]+\E_{z\thicksim p_{z}}\Big[\log \Big(1-D_1\big(G(z)\big)\Big)\Big]\notag\\
    &\qquad \qquad +\E_{x\thicksim p_{\text{data}}}\big[\log D_2(x)\big]+\E_{z\thicksim p_{z}}\Big[\log \Big(1-D_2\big(G(z)\big)\Big)\Big].
\end{align}
Since we assume enough capacity, the inner max game is achieved if and only if:
$D_1(x)=D_2(x)=\frac{p_{\text{data}}(x)}{p_{\text{data}}(x)+p_g(x)}$. To prove $p_g$ converges to $p_{\text{data}}$, only need to reproduce the proof of proposition 2 in \cite{gan}. We omit the details here.

\end{proof}

\section{Experiment Details and Additional Results}\label{app:exp}
\paragraph{Model Architectures}

For the small-scale datasets, we used a shallow version of generator and discriminator: three convolution layers in the generator and four layers in the discriminators. We use a deep version of generator and discriminator for natural scene and human face image generation, which have three convolution layers in the generator and seven layers in the discriminators. The deep version is the original design of DCGAN\cite{DCGAN}. The peer discriminator uses the duplicate version of the first one. 

\subsection{Architecture Comparison Between GAN, D2GAN and \PG{}}
Figure \ref{Fig:compare_arch} shows the architecture designs of single discriminator, dual discriminator, and our proposed \PG{}. Compared with Vanilla GAN, \PG{} has one more identical discriminator and a competitive Duel Game between two discriminators. The introduced Duel Game induces diversified generated samples by discouraging the agreement between $D_1$ and $D_2$. In D2GAN, although both discriminators are trained with different loss functions, they do not interfere with each other in the training. 
\begin{figure*}[!htb]
\vspace{-0.07in}
    \centering
    \vspace{-0.1in}
    {\includegraphics[width=0.72\textwidth]{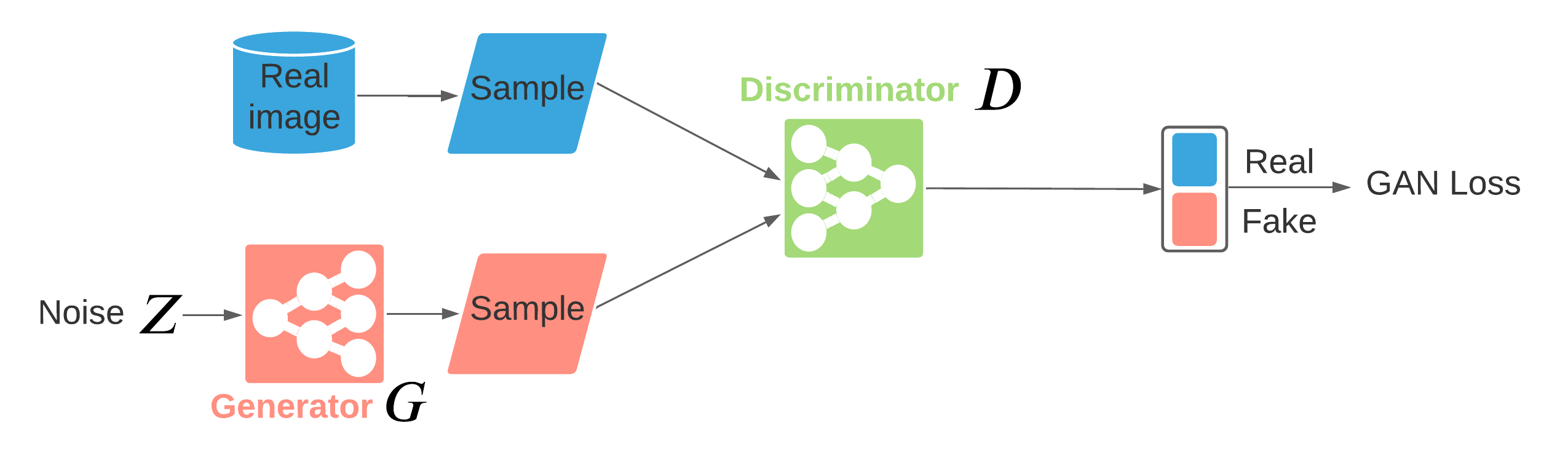}
    }
    \vspace{-0.1in}
    {\includegraphics[width=0.72\textwidth]{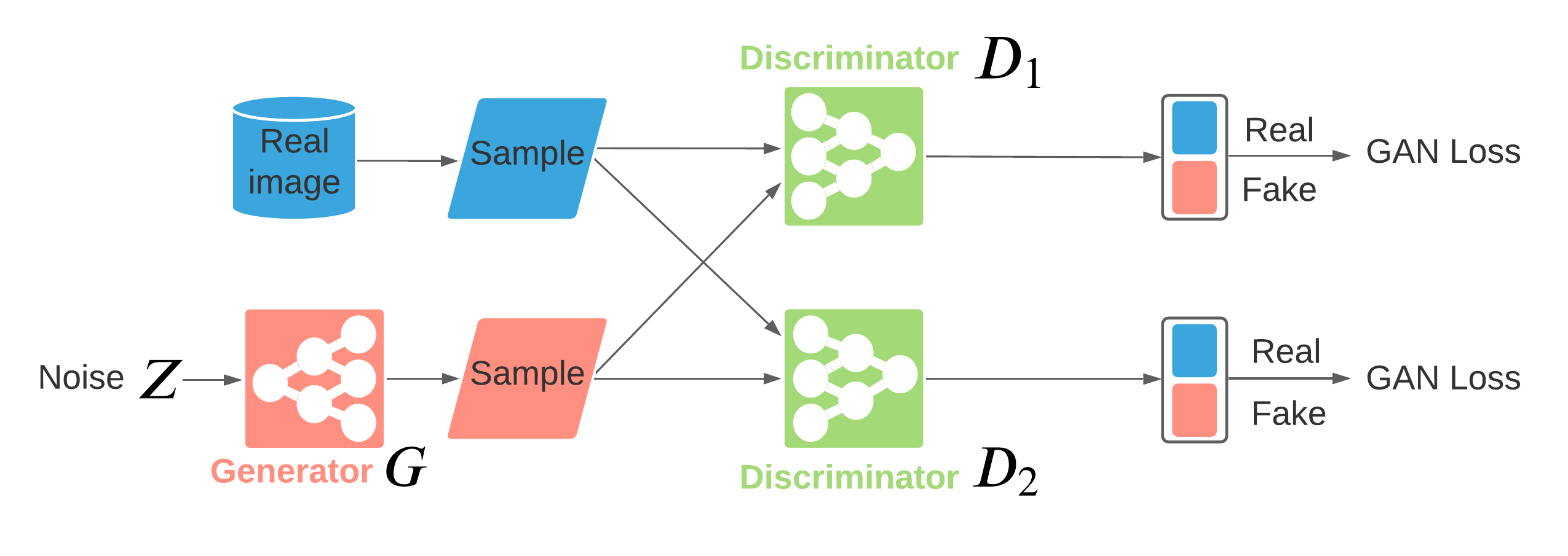}
    }
    {\includegraphics[width=0.72\textwidth]{figures/flowchart_peergan.pdf}
    }
        \vspace{-5pt}
        \caption{Architecture comparisons between GAN based method (first row), dual discriminators GAN based method (second row) and \PG{} (third row).}
    \label{Fig:compare_arch}
\end{figure*}

\subsection{Additional Experiment Results}

StyleGAN-ADA \cite{karras2020training} is the state-of-the-art method in image generation. We applied our duel game to StyleGAN-ADA and further improves its performance. On CelebA \cite{CelebA} dataset, we improved FID from 4.85 to 4.52, and FFHQ-10k\cite{karras2019style} dataset improved FID from 7.24 to 6.01. We show the generated image results (trained on CelebA) in Figure \ref{Fig:celebA_sup}.
 \begin{figure}
 \vspace{-0.1in}
    \centering
    {\includegraphics[width=1.0\textwidth]{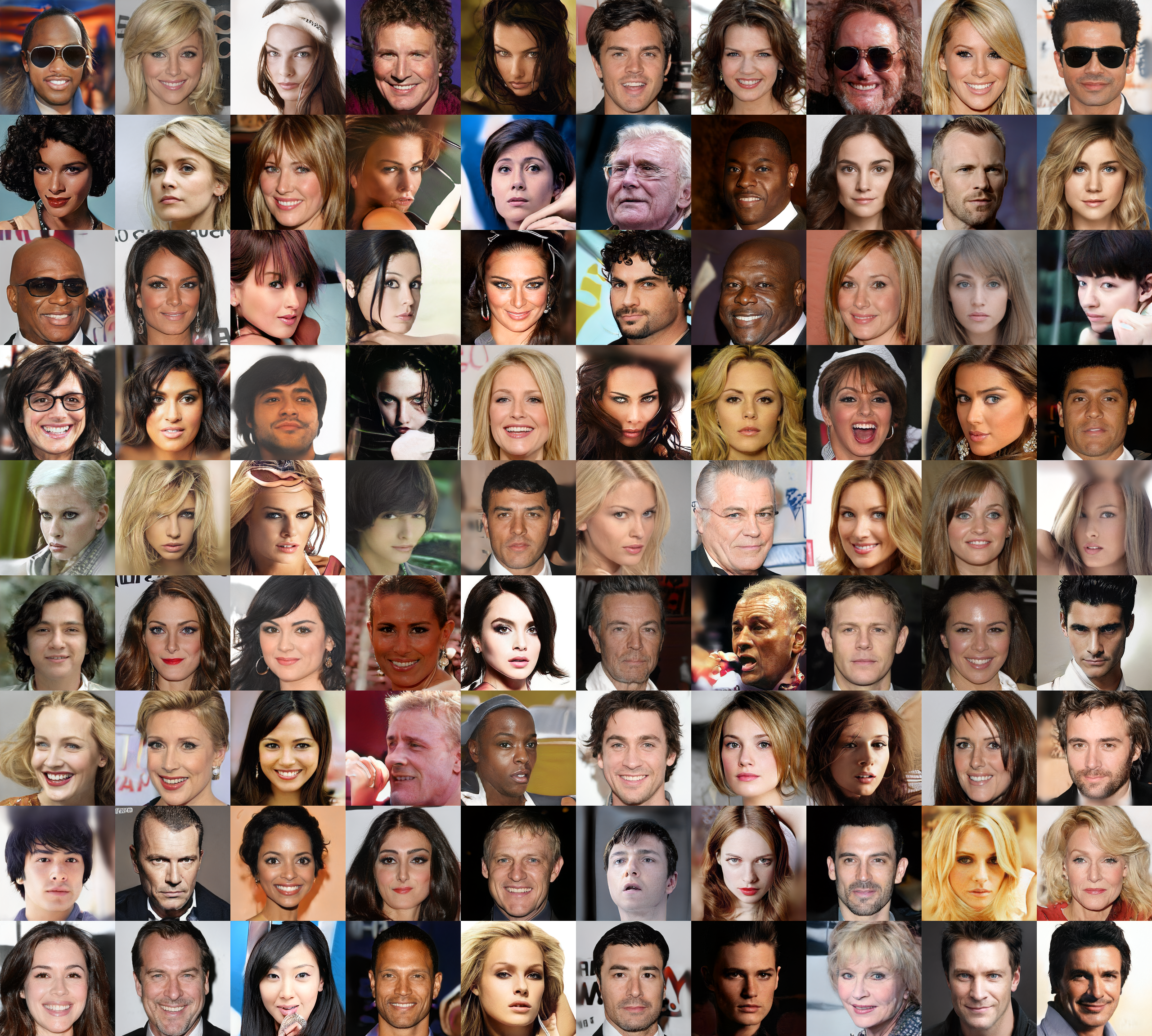}
    }
 \vspace{-0.15in}
    \caption{More CelebA image generation results of applying duel game on StyleGAN-ADA.}
 \vspace{-0.15in}
    \label{Fig:celebA_sup}
\end{figure}

\subsection{Additional Experiment Details}

\paragraph{Model Architectures}

For the small-scale datasets, we used a shallow version of generator and discriminator: three convolution layers in the generator and four layers in the
discriminators. We use a deep version of generator and discriminator for natural
scene and human face image generation, which have three convolution layers in
the generator and seven layers in the discriminators. The deep version is the
original design of DCGAN\cite{DCGAN}. The peer discriminator uses the duplicate version of the first one.

\paragraph{Hyper-Parameters}\label{app:para} 
 
\PG{} achieves low FID scores and high IS scores when $\alpha$ and $\beta$ are simply set to constant values. However we found that we could obtain an approximately 10\% improvement through dynamic tuning. The parameter $\beta$ controls the overall weight of $\dd$, while $\alpha$ punishes the condition when $D_1$ over-agrees with $D_2$. In the early training phase when we have an unstable generator and discriminator, we set $\alpha$ and $\beta$ to 0. As training progresses, we gradually increase these parameters to a max value, which helps with vanishing gradients. After the midpoint of training we decrease these parameters to help the discriminators converge, until the parameters reach approximately 0 at the end of the training process. We adopt 0.3, 0.5 as the max value for $\alpha$ and $\beta$, respectively.

\subsection{Ablation Study of \PG{}}
During training, We initialize the $\alpha$ and $\beta$ as 0, and gradually increase to the set maximum value. We experimentally discover $\alpha$=0.3 and $\beta$=0.5 can achieve the best FID score in the datasets we tested on. Table 3 shows an thorough ablation of  different hyper-parameter setting on STL-10 dataset. The bold text are the best $\alpha$ setting when beta is fixed. 
 
 \begin{figure}
    \centering
    {\includegraphics[width=0.5\textwidth]{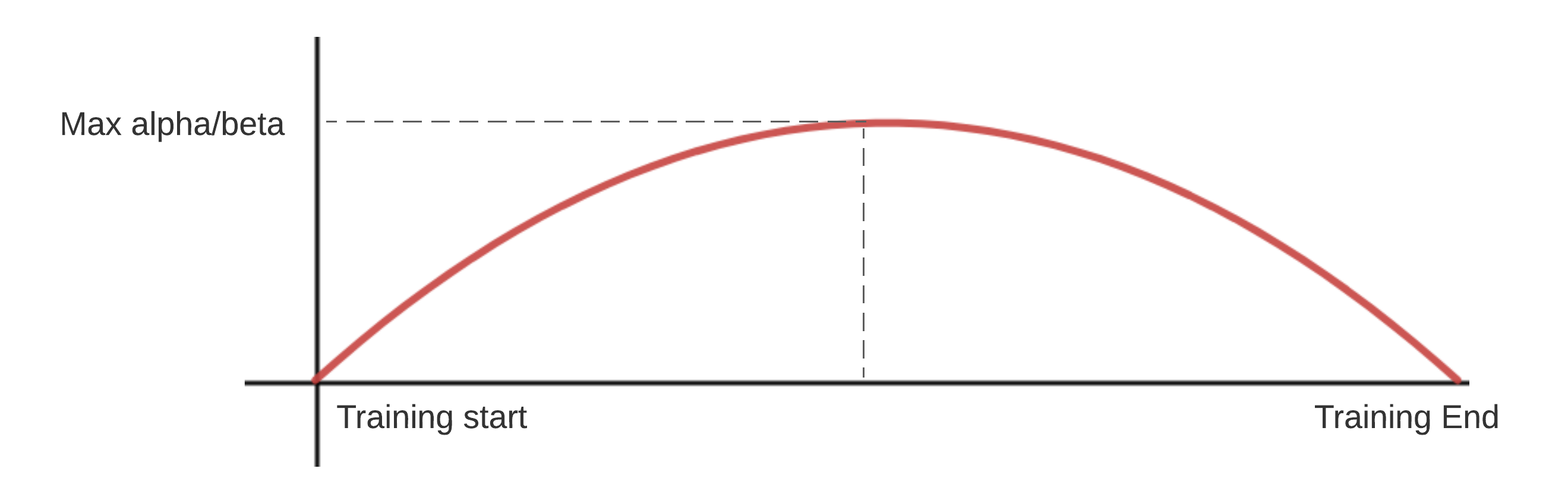}
    }
    \caption{The trend of $\alpha, \beta$ in the training.}
    \end{figure}
\begin{table}
    \centering
  \vspace{-0.3in}
\begin{tabular}{l|cccccc}
% \small
\hline
            &$\alpha$=0.1    & $\alpha$=0.3   &$\alpha$=0.5   & \:\:\:\:$\alpha$=0.7\:\:\:\:\: &\:\:$\alpha$=0.9\:\:   \\ \hline
$\beta$=0.25	& 60.88 & 56.01	& \textbf{51.86} & 58.17 & 60.91 \\
$\beta$=0.50	& 58.77 & \textbf{51.37} & 58.45 & 55.16 & 57.75 \\
$\beta$=0.75	& \textbf{55.07} & 59.58 &58.58 & 58.22 & 57.75 \\ \hline
\end{tabular}
\caption{Ablation study of max $\alpha$ and max $\beta$ value tuning on STL-10 dataset (evaluate with FID score).}
\end{table}

\subsection{Stability of Training}\label{app:stab}
In this section, we empirically show the stability of \PG{} training procedure. We adopt STL-10 dataset and $\beta=0.25$ for illustration. In Figure \ref{Fig:loss1} and \ref{Fig:loss2}, we visualize the loss of two discriminators during the training procedure of STL-10 dataset. The red lines indicate the smoothed trend of the loss evaluated on the generated images and real images. Real losses are represented by the shaded red lines. Although there exists certain unstable episodes (the difference between smoothed loss and the real loss is large) for both discriminators, the overall trend of both discriminators are stable. What is more, we do observe that $D_1$ and $D_2$ hardly experience unstable episodes at the same time. This phenomenon further validates our conclusion in Theorem 2: an unstable/diverged discriminator hardly disrupts the training of its peer discriminator!

\begin{figure*}[!htb]
\vspace{-0.15in}
    \centering
    % \vspace{-0.1in}
    {\includegraphics[width=0.3\textwidth]{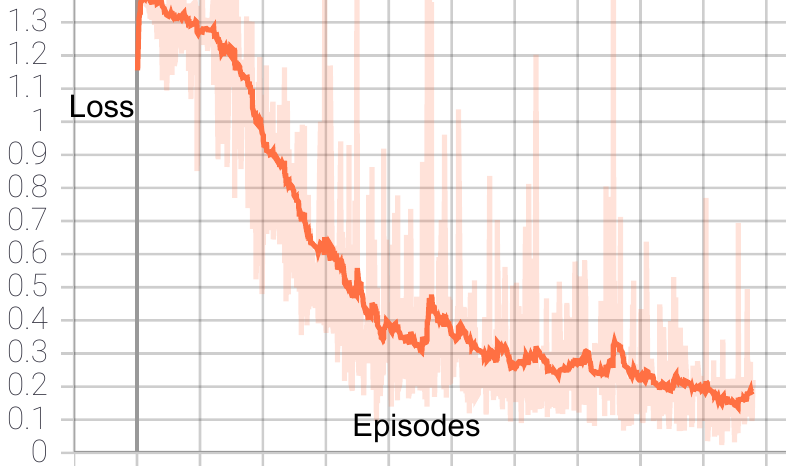}
    }
    % \vspace{-0.1in}
    {\includegraphics[width=0.3\textwidth]{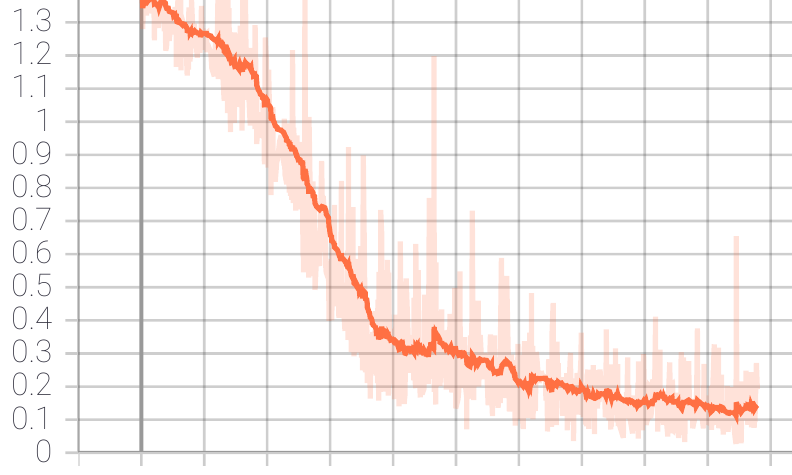} 
    }
    {\includegraphics[width=0.3\textwidth]{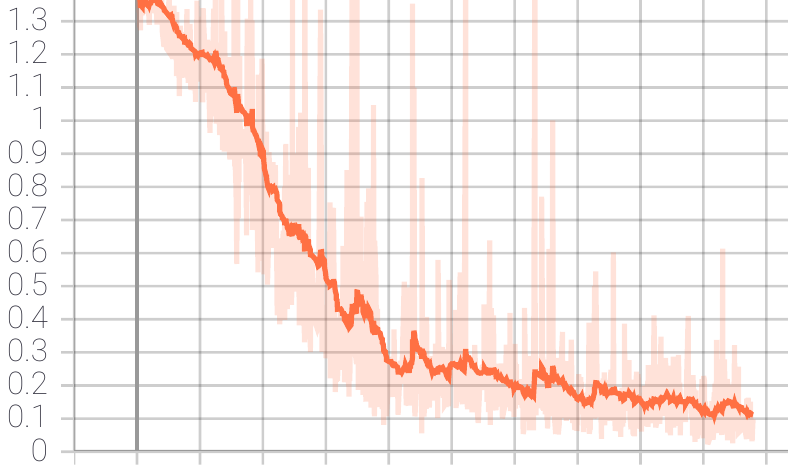}
    }
        \vspace{-5pt}
        \caption{The loss of $D_1$ in \PG{} with $\beta=0.25$ on STL-10 dataset, left: $\alpha=0.3$; middle: $0.5$; right: $\alpha=0.7$.}
    \label{Fig:loss1}
\end{figure*}

\begin{figure*}[!htb]
\vspace{-0.15in}
% \vspace{-0.07in}
    \centering
    % \vspace{-0.1in}
    {\includegraphics[width=0.3\textwidth]{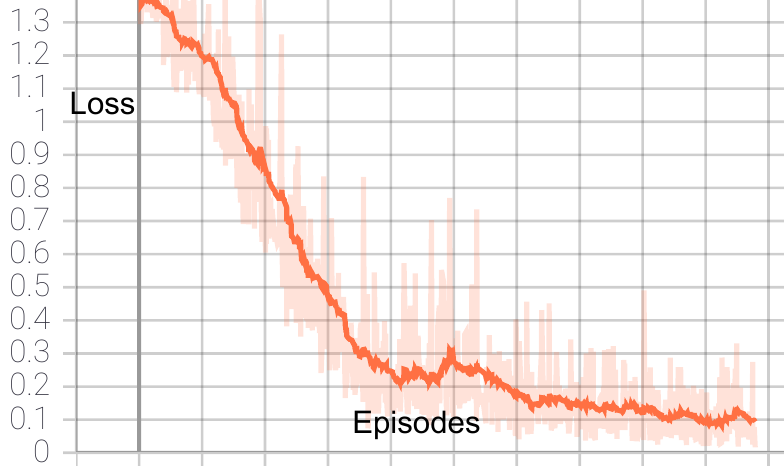}
    }
    % \vspace{-0.1in}
    {\includegraphics[width=0.3\textwidth]{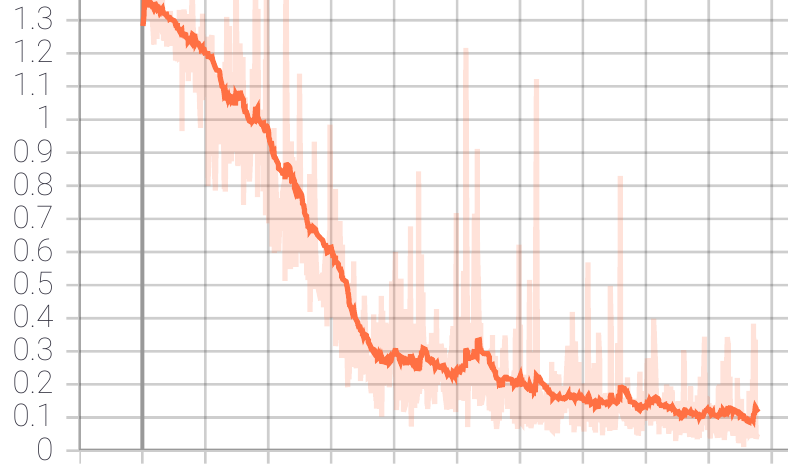} 
    }
    {\includegraphics[width=0.3\textwidth]{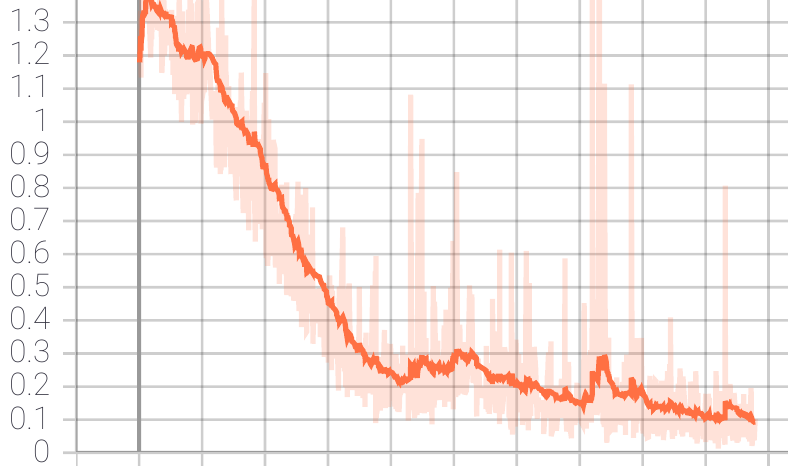}
    }
        \vspace{-5pt}
        \caption{The loss of $D_2$ in \PG{} with $\beta=0.25$ on STL-10 dataset, left: $\alpha=0.3$; middle: $0.5$; right: $\alpha=0.7$.}
    \label{Fig:loss2}
\end{figure*}

\paragraph{\textbf{Agreements Between Two Discriminators}}
We also empirically estimate the agreement level between two discriminators while training. In Figure \ref{Fig:agree}, the $y-$axis denotes the percentage of predictions that reach a consensus by $D_1$ and $D_2$. The smoothed curve depicts the overall change of the agreement level. At the initial stage, $D_i$ is not encouraged to agree overly on its peer discriminator $D_j$. As the training progresses, the agreement level gradually increases to a high value to help the convergence of the whole training process. The shaded red line means that the practical agreement level fluctuates around the smoothed line, incurs a certain degree of randomness and prevents discriminators from getting stuck in a local optimum.

\begin{figure*}[!htb]
\vspace{-0.15in}
    \centering
    {\includegraphics[width=0.3\textwidth]{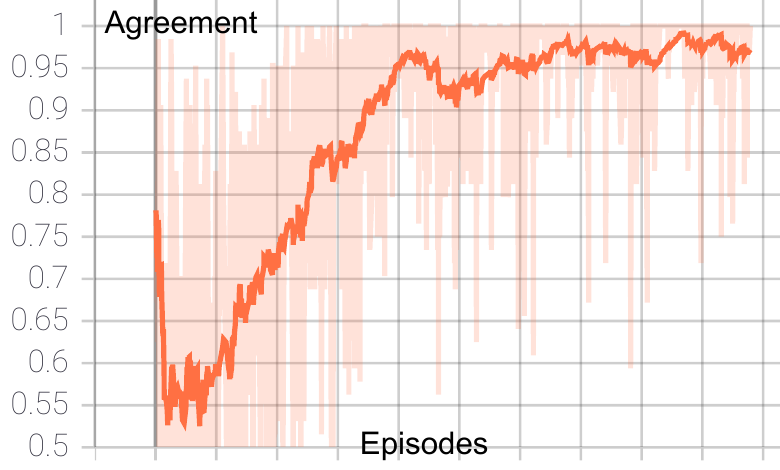}
    } 
    {\includegraphics[width=0.3\textwidth]{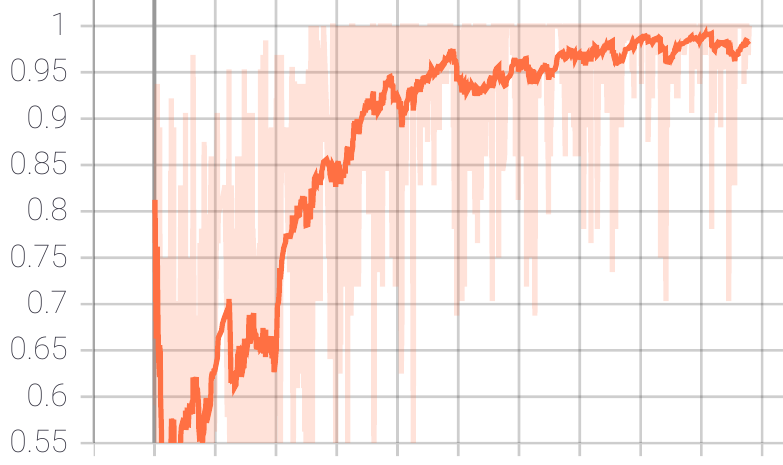} 
    }
    {\includegraphics[width=0.3\textwidth]{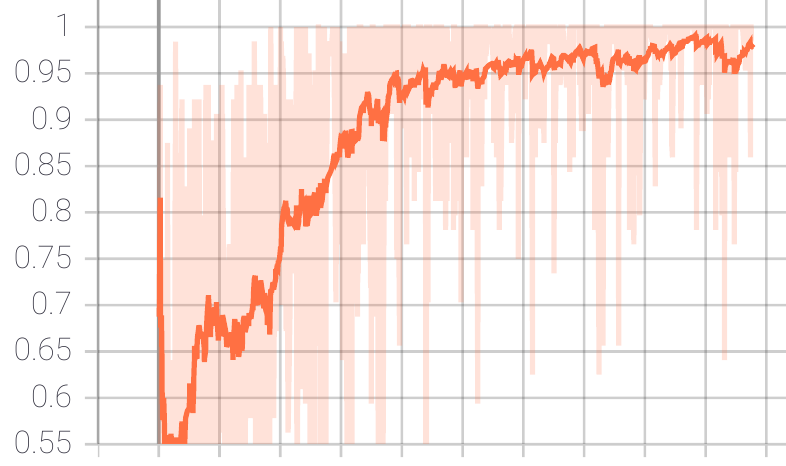}
    }
        \vspace{-5pt}
        \caption{The agreement level between $D_1$ and $D_2$ in \PG{} with $\beta=0.25$ on STL-10 dataset, left: $\alpha=0.3$; middle: $0.5$; right: $\alpha=0.7$.}
    \label{Fig:agree}
\end{figure*}

\end{document}